\documentclass[twocolumn]{svjour3}
\smartqed  
\usepackage[utf8]{inputenc}

\usepackage{graphicx}
\usepackage{multirow}

\usepackage{appendix}
\usepackage{color}

\newcommand{\tabref}[1]{\mbox{Table~\ref{tbl:#1}}}
\newcommand{\figref}[1]{\mbox{Figure~\ref{fig:#1}}}

\begin{document}


\title{Semantic constraints to represent common sense required in household actions for multi-modal Learning-from-observation robot}
\author{K. Ikeuchi \and N. Wake \and R. Arakawa \and K. Sasabuchi \and J. Takamatsu}
\institute{K. Ikeuchi \and N. Wake \and R.Arakawa \and K. Sasabuchi \and J. Takamatsu \at Microsoft, One Microsoft Way, Redmond, WA, USA \\ \email{katsuike@microsoft.com}}
\date{January 2021}

\maketitle

\section{Introduction}

Recent years have witnessed an increasing demand of service robots that can assist the elderly. Many elderlies now reside in senior residences that face support staff shortages. Further, although their residences are comfortable and well supported, the elderly still desire to live in their own homes when possible. To fulfill such requirements of living in their own homes necessitate require further support. Therefore, it is important and imminent to develop service robots that support the lives of the elderly in senior residences and/or in their own homes to meet these requirements.

The paradigm of learning-from-observation (LfO) is a promising direction toward this goal. An LfO system observes human actions and learns how to perform these actions via the observations. In senior residences and senior homes, we assume that nurses and care workers are stationed or visiting on a part-time basis.
These novice users can teach and/or modify the robot's action through the LfO system instead of manual programming.
Further, even though each home has a large variation in the environment, if necessary, such care takers can tune up the robot actions through their demonstrations on-site to absorb the environmental variations.

Thus far, most LfO systems have been developed for relatively clean environments such as machine assembly in industrial settings as done by \cite{ikeuchi1994toward} or rope handling in laboratory settings as done by \cite{Takamatsu2006}. The home environment is cluttered and household actions
have wide variations that require common sense to understand and pursuit the actions.
To overcome this cluttered environment, \cite{wake2020verbal} proposed a verbal-based focus-of-attention mechanism to guide the system's attention to the places where the key actions occur. This study overcomes the remaining issue of understanding household actions
that require common sense.

To understand household actions,
it is necessary not only to observe the demonstration, but also to have common sense that arises out of the purpose of the actions. For example, in the case of wiping a window,
it is necessary to have the common sense that the mop must maintain contact with the window surface through the wiping action; if only considered the possible physical states between the tool (mop) and the environment (window) and without common sense, the mop may detach and move freely away from the surface.

Previous research in LfO have mainly focused on only the observable physical constraints between the tool and the environment. We incorporate common sense to the paradigm by describing common sense explicitly using semantic constraints.
For example, in the previous wiping example, the wiping action is described by introducing a semantic wall in parallel with the wiped surface, and the mop can be only moved between these two surfaces: one physical and the other semantically defined imaginary surface.

In this paper,
we first analyze the
state transitions that occur because of translation and rotation displacements between
a tool and an environment
using screw theory. This analysis provides
the necessary and sufficient set of contact states and state transitions
to understand the types of semantic and physical constraints that are involved in household actions.
We further
examine real household videos and
classify the frequently occurring semantic and physical constraint 
sequential patterns into task groups.

Finally, we propose a way of recognizing task groups in an LfO system using textual/verbal input as a source of information on the semantic constraints.
Physical constraints can be observed from demonstrations; however, semantic constraints are difficult, if not impossible, to obtain from the observation of demonstrations. Therefore, we assume that textual or verbal input provides hints to obtain such semantic constraints. As preliminary experiments in this direction, we conducted experiments on a house cleaning domain to determine whether such textual input can 
achieve the recognition of task groups involving semantic constraints.

The contribution of this study are as follows:
\begin{itemize}
    \item To the best of our knowledge, this is the first paper that has thoroughly investigated the full set of necessary contact states and state transitions in translation and rotation displacements between a tool and an environment.
    \item A novel approach to represent common sense in household actions using semantic constraints is proposed, and the appropriateness of the representation is shown against top hit household YouTube videos as well as real home cooking recordings. Further, the frequent constraint patterns are organized into task groups.
    \item Preliminary results on the possibilities of recognizing task groups in an LfO system using textual input are shown to provide future directions for incorporating common sense in robot teaching.
\end{itemize}

The remainder of this manuscript is organized as follows.
Section 2 introduces related work and clarifies the objective of this paper. Section 3 establishes the basic description of states and state transitions of rigid objects.
Section 4 introduces semantic and physical constraints
from the state transitions. Section 5 
analyzes household actions using these constraint representations and extract
frequent-appearing constraint patterns as
task groups. 
Section 6 presents experiments to 
recognize task groups in an LfO system using textual input.
Section 7 concludes this paper.

\section{Related Works}
This study addresses the problem of conveying common sense within the context of robot teaching. 
We define common sense as a constraint on a motion necessary to achieve a manipulation; however, it is difficult to obtain this constraint directly from observation. 
Further, we propose the use of linguistic input to help recognize such semantic constraints.
The position of this study is explained via the review of previous robot teaching frameworks in terms of representing motion constraints and the use of language input. The review focuses on studies of object manipulation; non-manipulative applications such as navigation are beyond the scope of this study.

\subsection{Representation of motion constraints in robot teaching frameworks}
LfO is a robot teaching framework that aims to map one-shot demonstrations to robot motions through intermediate task representation, which is referred to as \textit{task model} and previously worked by \cite{ikeuchi1994toward,Takamatsu2006,Takamatsu2007,Nagahama2019,Perez-DArpino2017,Subramani2018}.
In a typical LfO system, a task is defined as the transition of a target object; for example, a contact state between polyhedral objects for part assembly by \cite{ikeuchi1994toward} or a topology of a string for tying a knot by \cite{Takamatsu2006}.
To extend the LfO to household actions,
\cite{wake2020learning} recently reported an LfO system that supports motion constraints derived from the linkage mechanism and the task representations were mapped on robots of various configuration by \cite{sasabuchi2020task}.
Although these previous systems achieved  success in specific task domains, the tasks were defined with physical constraints; semantic constraints were ignored. 
In theory, LfO can cover motions that involve semantic constraints provided a set of appropriate task models is assigned. However, to the best of our knowledge, there is no LfO framework that deals with semantic constraints. 

Learning from demonstration (LfD) and programming by demonstration are also popular frameworks for robot teaching. In this paper, we collectively refer to them as LfD.
The majority of LfD research aims to obtain a state-action pair, which is the so-called \textit{policy}, by learning repeated demonstrations as worked by \cite{Argall2009a,Billard2008,Schaal1999}.
Because human demonstration reflects the intent of the teacher, LfD systems can learn manipulations with semantic constraints (e.g., \textit{scooping things while avoiding spilling} by \cite{Akgun2012}) and physical constraints (e.g., \textit{block building} by \cite{Orendt2017}, \textit{rotating} by \cite{Liu2019a}, or \textit{scooping} by \cite{Akgun2012}).
However, because policy learning is based on the imitation of human movements, the learned manipulations do not accompany an explicit understanding of semantic constraints.
Although several LfD studies have incorporated robots' experiences by applying reinforcement learning as in the work by \cite{balakuntala2019extending,guenter2007reinforcement} or meta-learning by \cite{yu2018one}, they do not aim to learn semantic constraints. It seems that LfD systems do not learn the essence of the task (e.g., \textit{movement to avoid spilling}); instead, they learn movements that are fine-tuned to a scene by a human intention or robots' experiences.

In summary, although both LfO and LfD have the potential to deal with semantic constraints, they are based on different design concepts of intermediate task representation; LfO uses task models of state transitions, while LfD uses policies that learned repetitive human demonstrations. 
This study is based on LfO and proposes a solution that operates with an understanding of the purpose of tasks, by defining state transitions that consider semantic constraints.

\subsection{Use of language in robot teaching}
The problem of recognizing semantic knowledge from language is a form of ``symbol grounding.'' The accumulating evidence suggests that linguistic symbol grounding can solve a variety of problems in robotic applications. Examples include video object segmentation by \cite{Khoreva2018}, visual and verbal navigation by \cite{Anderson2019}, human-robot cooperation by \cite{Petit2013,Liu2019}, interactive learning by \cite{Chai2018,Mohan2012}, and bidirectional mapping between human movement and natural language by \cite{plappert2018learning}. By focusing on the function of language to provide semantic grounding, robot teaching applications have addressed linking an instructor's linguistic input to execution operations, such as pick-place by \cite{Forbes2015,Lueth2002}, grasping by \cite{Ralph2008,Wake2020grasp}, virtual-block relocation by \cite{Bisk2016}, and mobile manipulation  by \cite{Howard2014,Tellex2011}.

These studies support the idea that languages can provide promising hints for estimating semantic knowledge related to object manipulation. However, to the best of our knowledge, no attempt has yet been made to address semantic constraints in manipulation and use language to infer those constraints explicitly. Recent work closely related to our study is presented by \cite{Paulius2019,David2020}. In these studies, the authors defined a taxonomy of manipulation motions for cooking, called \textit{motion code}, and they related each motion code with verbs. Their proposed motion code is defined in terms of contact states and trajectories of motion, and covers a wide range of actions. However, they lump together actions with (e.g., \textit{pour}) and without semantic constraints (e.g., \textit{pick-and-place}). Unlike their approach, we propose an action class based on motion constraints. Further, we attempt to highlight the role of language in the recognition of semantic constraints by examining the correspondence between action classes and instructional texts, and not just that of between verbs. 

\section{Defining contact states in translation and rotation}

To analyze 
physical and semantic constraints,
admissible displacements of an object/tool (e.g., a mop) with respect to its environment (e.g., a floor)
are discussed using screw theory.  These admissible displacements characterize the state of the object in relation to the environment.
In the later sections, we will use this state characterization and transitions of these states to understand the  physical and semantic constraints involved in household actions.

\subsection{Characterizing admissible displacements}

The admissible displacements of a rigid body are constrained by rigid environment objects through point contacts. We assume that constraint points have neighborhoods, which are approximated as planar patches, thereby guaranteeing the differentiability at the contact points. Further, we do not consider singular cases such as two polyhedrons contacting an edge to another edge with each other because such singular cases rarely occur in the household action domains.

\paragraph{Constraint inequality equation}
An admissible displacement of a rigid body is formulated using screw theory by \cite{roth1984screws}. The screw theory formulates a translation displacement along the screw axis and a rotational displacement as a circular displacement around the screw axis. Any displacement constraint, provided by a contact point $\vec{P}$, is represented as a screw:

\begin{equation}
        \vec{N} \cdot \vec{T} + (\vec{P} \times \vec{N}) \cdot \vec{S}  \geq 0,
        \label{eq: screw}
\end{equation}
where $\vec{N}$ denotes the normal vector at the contact point and $\vec{S}$ denotes the screw axis vector. A translation displacement occurs along $\vec{S}$, and a rotation displacement occurs around $\vec{S}$. The ratio between the translation and rotation is specified by the parameter $p$.  $\vec{T} = \vec{C} \times \vec{S} + p \vec{S}$, where $ \vec{C} $ denotes the center of rotation. 

From this equation, we obtain
\begin{equation}
    (\vec{M} \cdot \vec{S}) + p (\vec{N} \cdot \vec{S}) \geq 0,
    \label{eq:screw}
\end{equation}
where $\vec{M} = \vec{Q} \times \vec{N}$ and $ \vec{Q} = (\vec{P} - \vec{C})$.

For a pure translation motion, $p=\infty$,
\begin{equation}
    \vec{N} \cdot \vec{S} \geq 0
    \label{eq:translation}
\end{equation}
is the constraint equation for the admissible axis directions given by a contact point. 

For a pure rotation motion, $p=0$, 
\begin{equation}
    \vec{M} \cdot \vec{S} \geq 0
    \label{eq:rotation}
\end{equation}
is the constraint equation.

In our latter analysis, we find that most household tasks can be represented as either pure translation or pure rotation motions at each moment. Exceptions include screw-closing actions; however, even in this case, the translation displacement is considerably smaller than the rotation displacement, and therefore, such actions can be considered as pure rotations for analytical purposes. We do not consider the combined case in this study. 

\paragraph{Gaussian sphere and axis directions}

We use a Gaussian sphere defined by \cite{gauss2005general,horn1986robot} as a tool to depict the screw axis. One-axial direction is represented as a three-dimensional unit vector $\vec{S}$. The Gauss mapping projects the starting point of the unit vector to the origin and end point on a point on the sphere surface (referred to as the Gaussian sphere), thereby resulting in any three-dimensional unit vector to be represented as a point on the Gaussian sphere. See Fig.~\ref{fig:gaussian-sphere}.

\begin{figure}
    \centering
    \includegraphics[width=\columnwidth]{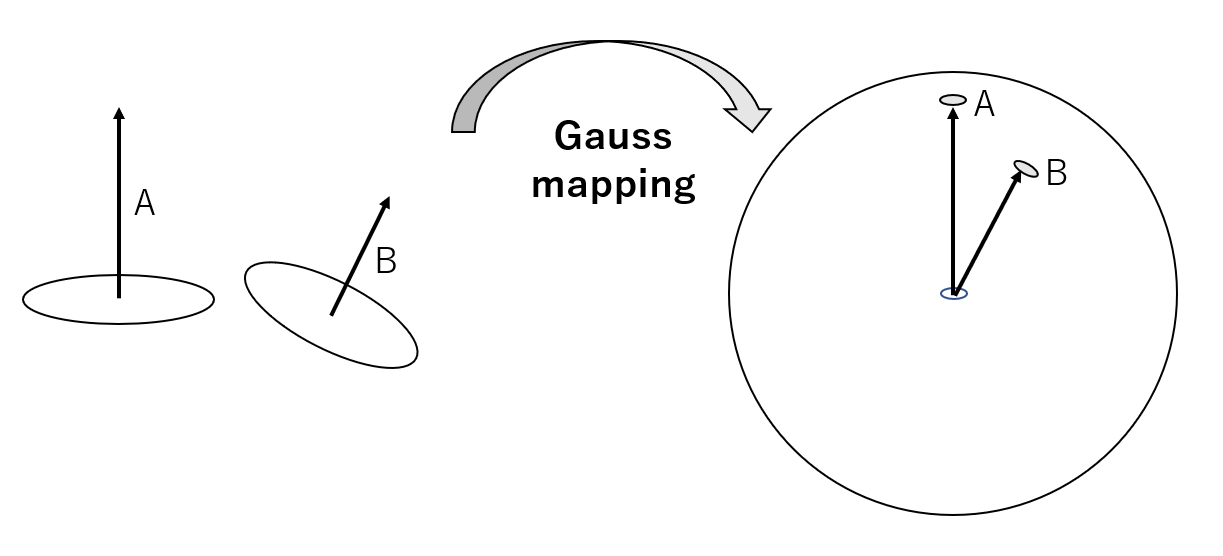}
    \caption{Gauss mapping and Gaussian sphere. A unit vector is mapped to a point on the Gaussian sphere. Its starting point and the end point are projected to the sphere center and spherical point, respectively. The point represents the vector. }
    \label{fig:gaussian-sphere}
\end{figure}

\paragraph{Pure translation}

In pure translation motion, any motion direction is aligned to the axis direction. Thus, in the later discussion on pure translations, we use an axis direction and a motion direction interchangeably. In particular, the spherical surface of the Gaussian sphere represents the space of all possible axis and transnational motion directions.

Let us consider the case when a rigid object is constrained by a point contact as shown in Fig.~\ref{fig:T1}(a). From Eq.(\ref{eq:translation}) with $p=\infty$,

\begin{eqnarray}
    \vec{N} \cdot \vec{S} & \geq & 0, \nonumber 
\end{eqnarray}
where $\vec{N}$ denotes the normal direction.

We can set the normal direction $\vec{N}$ as the north pole of the Gaussian sphere without loss of generality, as shown in Fig.~\ref{fig:T1}(b). Then, any admissible axis direction $\vec{S}$ satisfying Eq.~\ref{eq:translation} can be expressed as a point of the northern hemisphere of the Gaussian sphere, which is depicted as the white region in Fig.~\ref{fig:T1}(b). The gray southern hemisphere, which represents downward motions in Fig.\ref{fig:T1}(b), depicts nonadmissible axial directions.

An infinitesimal motion corresponding to a point on the northern hemisphere breaks this relationship. Among these motion components, the pure motion component is located along the north-pole direction. 

\begin{figure}
    \centering
       \includegraphics[width=\columnwidth]{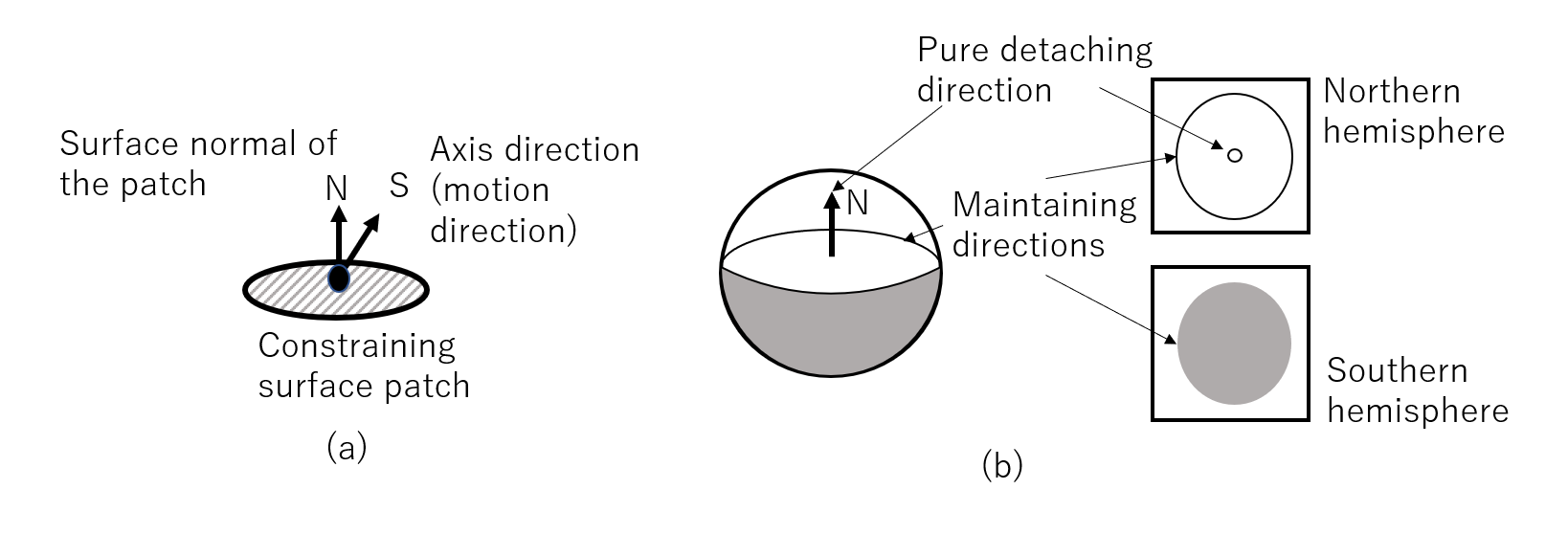}
       \caption{One directional contact. (a) A constraining surface patch has an infinitesimal neighboring region, wherein the normal direction is $\vec{N}$. The vector $\vec{S}$ denotes the one-axis direction (motion direction) of the rigid object $\vec{S}$ at the contact point. (b) Admissible and nonadmissible regions of the axes are depicted as white and gray hemispheres of the Gaussian sphere, respectively. The pure detaching direction is aligned to the north pole direction to break the contact, and the maintaining directions are depicted as equator points.}
    \label{fig:T1}

\end{figure}

When a rigid body is constrained by multiple contact points $i$, of which normal direction is depicted as $ \vec{N}_i $, the admissible axes, $ \vec {S} $ is the solution of the set of the simultaneous inequality that satisfies:
\begin{eqnarray}
\vec{N}_0 \cdot \vec{S} & \geq & 0 \nonumber \\
\vec{N}_1 \cdot \vec{S} & \geq & 0 \nonumber \\
\vdots & & \nonumber \\
\vec{N}_n \cdot \vec{S} & \geq & 0 
\label{eq:coef}
\end{eqnarray}

As shown in  equation~\ref{eq:coef}, we can collect all normal directions as a coefficient matrix $N$ of a simultaneous inequity equation. Depending on the rank of the matrix derived according to the Kuhn-–Tucker theory by \cite{kuhn1957linear}, this solution space can be classified into ten characteristic patterns with an entire spherical region, a hemispherical region, a crescent region, a polygonal region, an entire great circle region, a great semicircle region, an arc of a great circle, a pair of points, one point, and no region as the admissible axis regions. For the sake of simplicity, we group them into six groups, and define them as states and name those states as noncontact (NC), planar contact (PC), two-side planar contact (TR), one-way two-side planar contact (OT), prismatic-contact (PR), one-way prismatic contact (OP), and fully contact (FC) translation states, as shown in Figure ~\ref{fig:translation-contact}.

\begin{figure}
    \centering
  \includegraphics[scale=1.2]{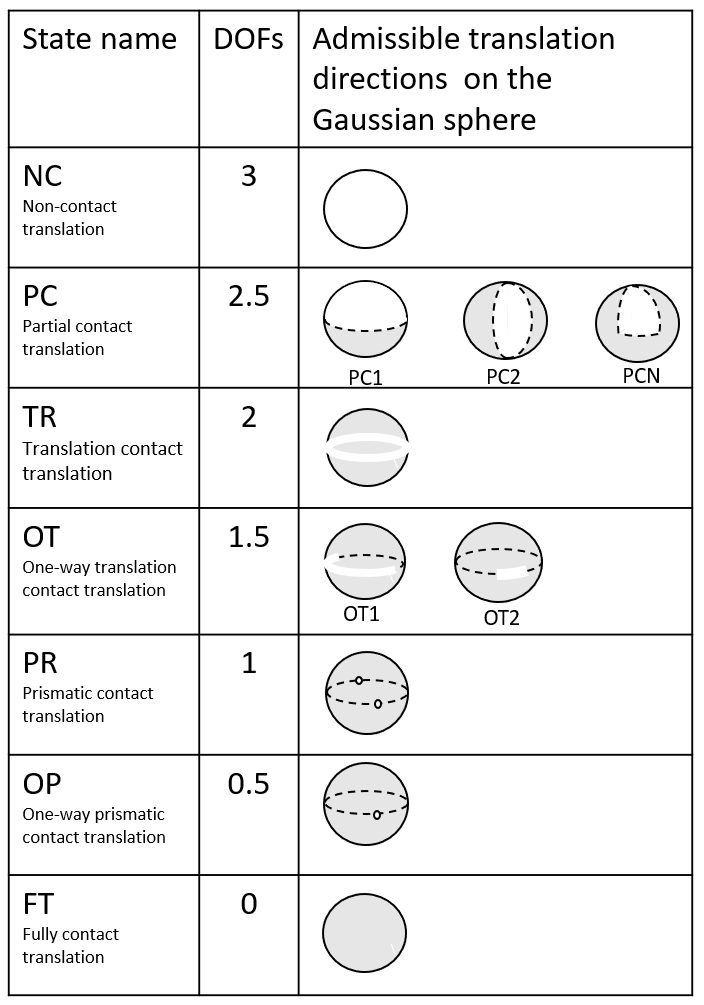}
    \caption{Seven translation contact states. For the sake of simplicity in the analysis, we grouped three partial contact states, a hemisphere, a crescent, and a polygonal shaped state, into one PC state, and two one-way states, a hemi-circle and an arc-shaped state, into one OT state.}
  
    \label{fig:translation-contact}
\end{figure}

\paragraph{Pure rotation}

For a pure infinitesimal rotation, the screw ratio $p$ becomes $0$, and then the equation is represented as: 

\begin{eqnarray}
\vec{M} \cdot \vec{S} & \geq & 0
\nonumber
\end{eqnarray}
where $\vec{M} = (\vec{Q} \times \vec{N})$ is the normal vector to the plane spanned by $\vec{Q}$ and $\vec{N}$.  

The motion constraint to the screw axis direction, $\vec{S}$, given by this point, varies depending on the value of $\vec{Q}$, which is the spatial relation between the rotation center and the contact point. When multiple contacts exist with respect to one object, its rotation is represented with respect to one common rotation center. Thus, we can set up simultaneous inequality equations using coefficients provided by known contact points and one common rotation center with respect to $\vec{S}$. Further, we can apply the Kuhn--Tucker theory and obtain 10 characteristic solution regions in the Gaussian sphere, as was the case in the pure translation.  

In the case of the finite rotation, the nonlinearity of the motion when admissible axis directions exist on the great circle in the infinitesimal analysis, are not a great circle but a combination of points and arcs that have been proved. See the Appendix for more detailed discussions and proof. In conclusion, we have 13 topological patterns of admissible axes on the Gaussian sphere: whole spherical surface, hemisphere surface, crescent region, polygonal region, a combination of arcs and points on a great circle, a pair of arcs on a great circle, two pairs of points on a great circle, a combination of arcs and points on a great semicircle, two points on a great semicircle, one arc on a great semicircle, two points on a great circle, one point, and no region.

For the sake of simplicity of the analysis, we group them into seven groups and named nonrotational contact (NR), rotational contact (RT), spherical joint (SP), one-way spherical joint (OS), revolute joint (RV), one-way revolute joint (OR), and fully constrained (FR) rotation states, as shown in Figure~\ref{fig:rotation-states}.

\begin{figure}
    \centering
    \includegraphics[scale=1.2]{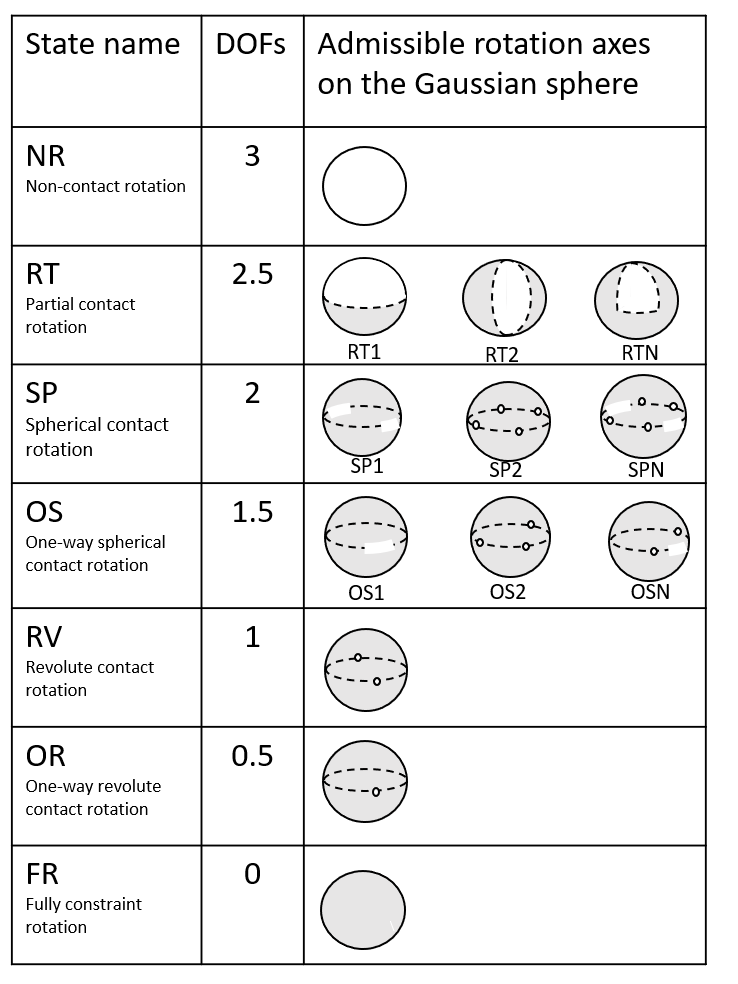}
    \caption{Rotation states.}
    \label{fig:rotation-states}
\end{figure}

\section{Physical and semantic constraints}

\subsection{Physical constraints and states}

We express the degree of freedom of an object by using the admissible axis directions of a screw represented as admissible regions on the Gaussian sphere. This representation is used to describe the constraints in household actions. 

Let us consider the example shown in Fig.~\ref{fig:draw}. In Fig.~\ref{fig:draw}(a), the drawer is pulled out halfway. This condition allows us either to pull it out or to push it a little bit along the drawing direction. The admissible axis directions in translation are represented as a pair of points on the Gaussian sphere: the PR translation state. The drawer cannot rotate around any axial directions, and there are no admissible axial directions on the Gaussian sphere: the FR rotation state. We can label the drawer state as the (PR,FR) state. 

In a completely closed state, as shown in Fig.~\ref{fig:draw}(b), the drawer can only move toward the opening direction. The admissible axis direction in translation is represented as one point on the Gaussian sphere, which is the OP translation state; it cannot rotate as well. Thus, we can label the drawer fully closed as the (OP, FR) state.

\begin{figure}[h]
    \centering
    \begin{tabular}{cc}
    
    \begin{minipage}{0.5\hsize}
    \centering
    \includegraphics[scale=1.0]{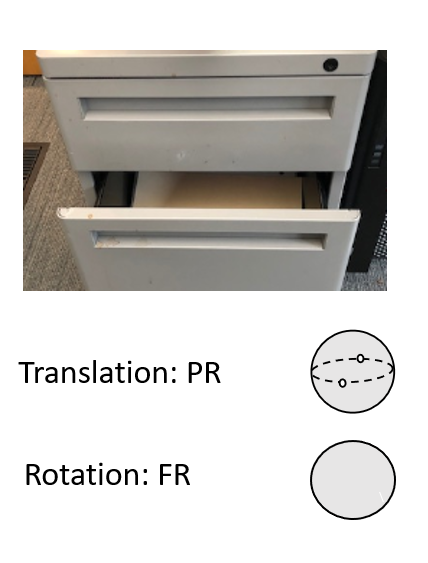}\\
    (a) 
    \end{minipage}
    
    \begin{minipage}{0.5\hsize}
    \centering
    \includegraphics[scale=1.0]{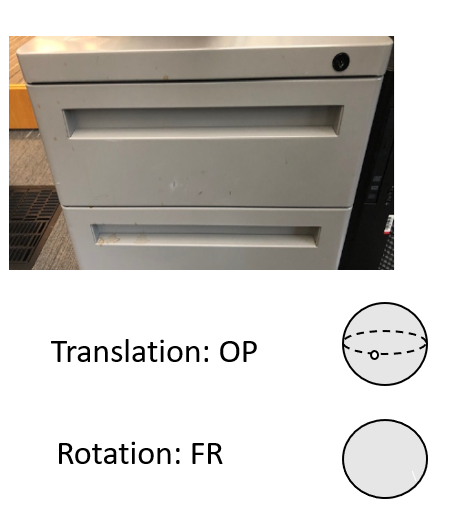} \\
    (b)
    \end{minipage}
    \end{tabular}
    \caption{States of a drawer. (a) Halfway open; One full DOF in translation: a pair of points in the Gaussian sphere, PR translation state. Zero DOFs in rotation: no admissible directions on the Gaussian sphere, FR rotation state (PR, FR). (b) Fully closed; Half DOFs in translation: one point on the Gaussian sphere, (OP) translation state. Zero DOFs in rotation, i.e., no point on the Gaussian sphere, FR rotation state, (OP, FR).}
    \label{fig:draw}
\end{figure}

Fig.~\ref{fig:statesintasks} expresses states of various household objects as the combinations of translation and rotation states.  Further, we can add other singular cases as shown in the figure. The degrees of freedom of a rigid body in contact with its environment object depends on where to take the position of the screw axis, i.e., the rotation center. Hereafter, unless the rotation center is specifically described, it is considered to be inside the convex hull of the object.

\begin{figure*}
    \centering
    \includegraphics{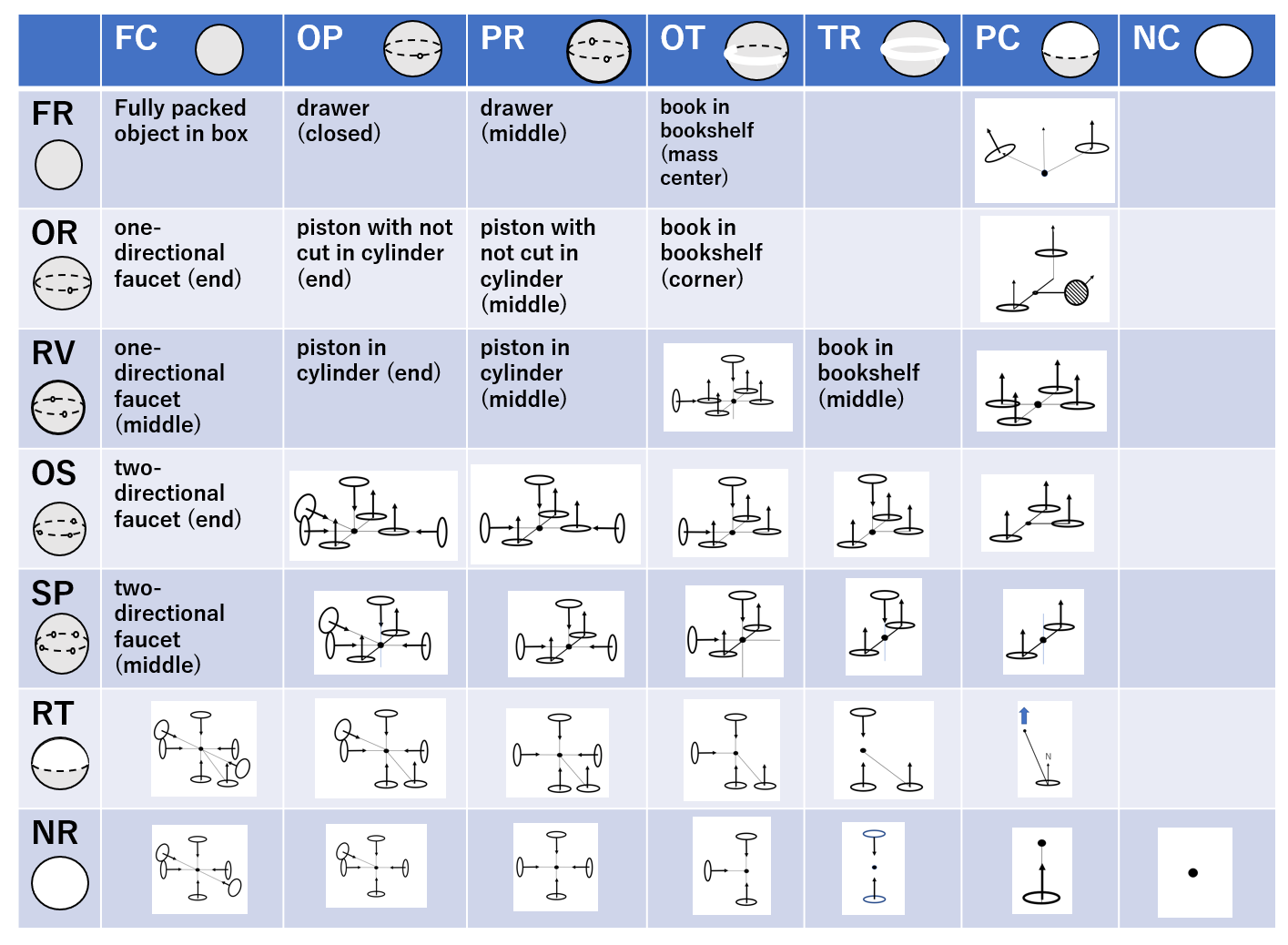}
    \caption{Various configurations of contact points and resulting states.}
    \label{fig:statesintasks}
\end{figure*}

\subsection{Semantic constraints and states}

In household
actions on top of the physical constraints provided by the environment, we need to consider semantic constraints. Semantic constraints express tips from the common sense, which are hidden human intentions. Let us consider an example of a cup filled with water. If we rotate this cup along the horizontal axis, the water will spill. Using our common sense, we usually avoid such spills; thus, adding a semantic constraint, given from the tip of ``it is not a good idea to spill out during the relocation of the cup,'' the possible axis direction of the cup in rotation during the relocation is a pair of points along the gravity direction, i.e., the RV (Revolute joint) contact state in the rotation. Thus, by considering the semantic constraint, the state of the cup is represented as the semantic state (NC, RV) instead of the original physical state (NC, NR).

Another example of semantic constraints is the state of a mop at the beginning of a floor-cleaning action. Under physical constraints, the state of the mop is represented as (PC, RV), as shown in Figure~\ref{fig:sem}(a). When one begins to clean the floor, the mop should not leave the floor or else we cannot clean the floor. Thus, the translation state under the semantic constraint is the TR. Moreover, because we will clean an inside area of the current mop position, it is only allowed to move a direction toward the center and not outside, i.e., the OT contact state. Thus, the mop state is represented as (OT, RV) instead of (PC, RV) because of the semantic constraints given by the common sense of a cleaning-floor action. See Figure~\ref{fig:sem}(b).

\begin{figure}
    \centering
    \includegraphics[width=\columnwidth]{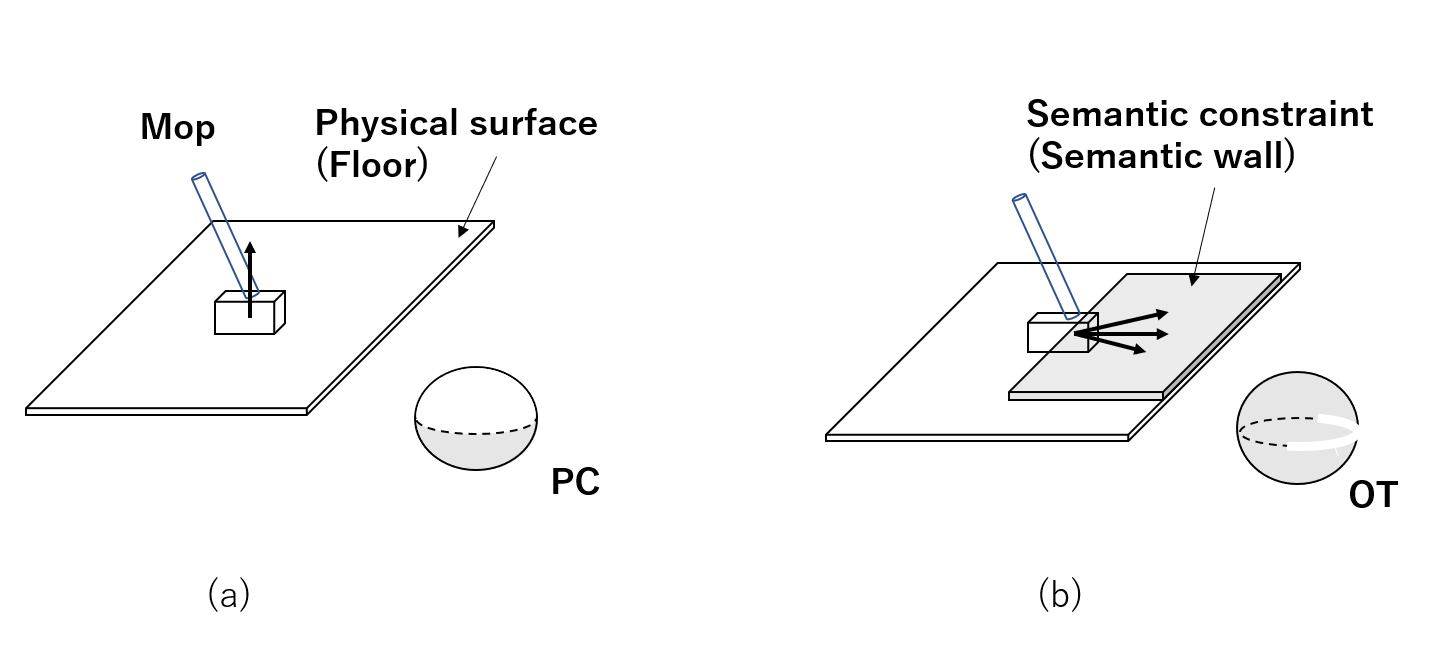}
       \caption{Physical and semantic constraints. (a) A mop contacts the floor. It can move upward, and the hemisphere is an admissible region; PC. (b) The mop cannot leave the floor surface because of the semantic constraints obtained through the common sense of ``cleaning-the-floor.'' The admissible region is an arc of a great circle, OT.} 
    \label{fig:sem}
\end{figure}

\section{Household action analysis and task groups}

A frequently appearing sequential pattern of physical and/or semantic constraints (i.e., fixed sequence of state transitions) is referred to as a task group.
These task groups can be used as an abstract concept which defines
what to do when designing and planning robot actions.
To determine the task groups in household actions,
we take two approaches. The first is to analyze household action in our own cooking videos directly and the second is to analyze YouTube captioning.

The first approach utilizes videos of making miso-soup and beef-stew from the preparation state to the serving stage, recorded by one of our collaborators with a careful selection of camera positions to avoid possible occlusions of hand motions. We then segment these videos into intervals surrounded by grasp-release pairs based on our hand speed segmenting program invented by \cite{yanokura} and describe those intervals using 
the state changes in physical and/or semantic constraints. 

The second approach is based on YouTube captioning. We collected 80 YouTube videos corresponding to cooking, carpet cleaning, floor cleaning, and furniture cleaning. Some of them are difficult to apply to our hand-speed segmentation program because of the blurred recording of hand movements. Thus, we instead collect the video captioning and extract verb-noun pairs. The action scenes associated with those verb-noun pairs provide state transitions. 

Finally, we merge the first and second results and extract
the task groups.
We can categorize the extracted task groups into three classes: physical, semantic, and multistage. Physical-class groups only need to consider transitions in physical constraints, while semantic-class groups consider their purpose of actions under semantic constraints. Multistage-class groups can be decomposed into multiple component groups defined in the previous two classes.

For numbering the following physical class, some numbers are intentionally omitted so that groups between physical and semantic classes correspond with each other. Moreover, some state transitions predicted in the screw theory rarely occur in household actions;
these are also omitted from the numbering.
In the following discussion, physical, semantic, and multistage groups are abbreviated as PTG, STG, and MTG, respectively.

\subsection{Physical class}

\paragraph{PTG1: Relocation group}

The first sequence pattern that frequently appeared in the data of household actions we used relate to relocating a target object, that is, picking, bringing, and placing. The state transitions involved are summarized in Table~\ref{tab:place}.

The {\it PTG11: Picking} task group lifts a target object, such as a box, from an environment, such as a table. At the beginning of the action, the box contacting the table surface can only translate in upward directions (PC state) and can rotate around the surface normal of the table (RV state). Thus, the box state begins with 
(PC, RV).
Immediately after the action of lifting breaks contact between the object and the environment; 
the box can translate in any direction (NC state) and can rotate about any axis (NR state). Thus, at the end of the action, the object state transitions to 
(NC, NR). 

In the data we used, it was rare for a single infinitesimal displacement to break the contact.
Instead, the displacement is followed by a finite displacement,
that is, NC-NC rather than a single NC, and NR-NR rather than a single NR. Therefore, the {\it Pick} action is more precisely a transition of translation states PC-NC-NC and rotation states RV-NR-NR. 

The {\it PTG12: Bringing} task group brings the target object from one location in the air to another location in the air. This group was often surrounded by the {\it Picking} and/or {\it Placing} group. 

The {\it PTG13: Placing} task group, which is an inverse of the previous {\it Picking} group, places an object on the environment surface. Initially, the object is in the air,
beginning with a
(NC, NR) state. Once the object contacts the environment surface, the movable direction of the object is in an upward direction, and the rotation axis is constrained along the surface normal: it is thus in the (PC, RV) state.

\begin{table}[h]
    \centering
    \caption{PTG1: Relocation task group}
    \label{tab:place}
    \begin{tabular}{|l||c|c|}
    \hline
    Task Group & Translation &  Rotation\\
     \hline
    PTG11: Picking & PC-NC-NC & RV-NR-NR \\
       \hline
   PTG12: Bringing & NC-NC & NR-NR \\
       \hline
      PTG13: Placing &  NC-NC-PC & NR-NR-RV \\
        \hline 
    \end{tabular}
    \end{table}

\paragraph{PTG3: Drawer task group}
This pattern
pulls out a target object, such as a drawer, from an environment, such as a chest,
When the drawer is closed, the physical state of the drawer is an OP translation state: it can only move in the direction of the opening. In the middle, it can move in either the open or close direction: it is in the PR translation state. Once fully opened, it can only move in the direction of closing, OP. As for the rotation, the drawer
cannot rotate in the FR rotation state. 
The state transitions are summarized in Table~\ref{tab:drawer}.

\begin{table}[h]
    \centering
    \caption{PTG3: Drawer task group}
    \label{tab:drawer}
    \begin{tabular}{|l||c|c|}
    \hline 
       Task Group & Translation &  Rotation \\
       \hline
      PTG31: DrawerOpening & {OP-PR-PR} &  \\
        \cline{1-2}
       PTG32: DrawerAdjusting & PR-PR & FR-(FR)-FR \\
        \cline{1-2}
        PTG33: DrawerClosing & PR-PR-OP & \\
        \hline 
    \end{tabular}
    \end{table}

\paragraph{PTG5: Burner task group} This pattern
rotates a target object, such as a burner switch, for which the rotation center is physically fixed. The rotation state
transitions according to the
opening, adjusting,
or closing phase of the switch. For example, 
in the opening/closing phase, the switch rotates only in the opening/closing direction and
is in the OR rotation state. In the adjusting phase,
the switch can rotate in either direction, that is, in the RV rotation state.
The rotation center is physically fixed, i.e., it is in the FC translation state.
Household actions such as opening a door also belong to this group by
considering that the center of rotation is at the hinge of the door.
The state transitions are summarized in Table~\ref{tab:turn}. 

\begin{table}[h]
    \centering
    \caption{PTG5: Burner task group}
    \label{tab:turn}
    \begin{tabular}{|l||c|c|}
    \hline 
       Task Group & Translation &  Rotation \\
       \hline
      PTG51: BurnerOpening & \multirow{3}{*}{FC-(FC)-FC} & OR-RV-RV \\
        \cline{1-1} \cline{3-3}
      PTG52: BurnerAdjusting &  & RV-RV \\
        \cline{1-1} \cline{3-3}
      PTG53: BurnerClosing &  & RV-RV-OR \\
        \hline 
    \end{tabular}
    \end{table}

\subsection{Semantic class}
Household actions require many tips that cannot be expressed by physical constraints and are obtained from semantic constraints given by the common sense involved in performing household actions. 

\paragraph{STG1: RelocationCarefully task group}

From the many household actions requiring semantic constraints that appeared in the data we used, the first pattern
is to pick, bring, or place a cup with liquid inside. If we ignore whether the liquid spills or not, the cup can be rotated freely during the relocation in the air. That is, in the NC rotation state. However, the common sense for relocating such a cup to avoid spilling includes the tacit knowledge that
you need to carry the cup so
that the normal direction of the liquid surface is aligned along the direction of gravity. 

We will refer to this semantic constraint of aligning the direction of an object to the direction of gravity as a {\it semantic ping} constraint; we should not rotate the cup about the axis perpendicular to the direction of gravity as if there is an imaginary ping standing from the water surface. See Figure\ref{fig:pingwall}(a). This constraint is semantically represented as the RV rotation state, instead of the NR. In terms of translation motions of the cup, there is no difference between carrying an empty cup and a cup with liquid. Therefore, these state transitions can be summarized as shown in Table~\ref{tab:bgc}.

\begin{table*}[h]
    \centering
    \caption{STG1: RelocationCarefully task group }
    \label{tab:bgc}
    \begin{tabular}{|c|l||c|c|}
    \hline 
      Semantic constraint  &Task Group & Translation & Rotation \\
       \hline
       \multirow{3}{*}{Semantic ping}
    
     &  STG11: PickingCarefully  & PC-NC-NC &  \\
        \cline{2-3}
   & STG12: BriningCarefully & NC-NC & RV-(RV)-RV \\
       \cline{2-3}
     &  STG13: PlacingCarefully & NC-NC-PC &  \\
       \hline
    \end{tabular}

\end{table*}

\paragraph{STG2: Wiping task group}
This pattern involves the planar translation of a tool on the environment surface, such as a table or a window. While the tool has a PC state in translation under physical constraints, the cleaning action cannot be achieved when the tool is detached from the environment surface. To avoid such detaching motion, we need to introduce 
a {\it semantic wall} parallel to the environment surface to make the contact state TR instead of PC in the translation state.
See Figure~\ref{fig:pingwall}(b).
To start/end the wiping action, the tool should move
in one direction,
and thus, it should have a OT translation state at the beginning or end. 

As for the rotation, small rotations around the surface normal of the environment are allowed for wiping a floor. However, since this rotation movement is not an intentional one,
we will simplify the rotation state and represent it as a FR state.

We can further divide this task group into three task groups, Wiping-start (OT-TR-TR), Wiping-middle (TR-TR) and Wiping-end (TR-TR-OT), as was the case to divide STG1 (RelocationCarefully group) into 
STG11, STG12, and STG31.
However, in the granularity of the semantic task groups, 
and from the analyzed video captions,
it should be rare for a human demonstrator
to use such fine detailed instructions. Thus, in this paper, we use
STG2 instead of STG21, STG22 and STG23
for the sake of simplicity. If necessary, of course, it is possible to use STG21, STG22 and STG23
for further detailed analysis. The similar argument can be applied in the following semantic task groups. See Table~\ref{tab:semsum}.

\begin{figure}[h]
    \centering
    \includegraphics[width=\columnwidth]{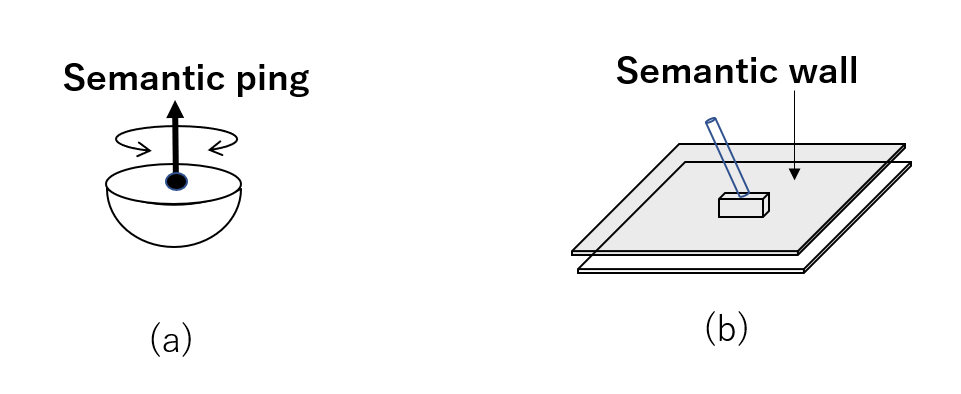}
    \caption{Semantic ping and wall. (a) Semantic ping. (b) Semantic wall.}
    \label{fig:pingwall}
\end{figure}

\paragraph{STG3: Peeling task Group}

In this pattern,
a tool, such as a peeler, needs to make a linear translation on the environment surface, such as a radish, which we refer to as the {\it semantic tube} constraint. While the tool has a PC translation state under physical constraints, the peeling action can be achieved only when the peeler moves along an imagery tube parallel to the radish surface; the peeler's state can be described as the PR translation state. To start and end the task, the peeler should move 
in one direction:
it should have the OP translation state.
For the rotation, rotations during peeling are not allowed; its rotation state is in the FR state. See Figure~\ref{fig:peeling-group}(a) and Table~\ref{tab:semsum}.

We found two minor variations of this {\it STG3: Peeling} group in the data we used.
The first variation, {\it  STG3b: AddingKetchup} occurs in actions such as sprinkling ketchup in the bottle. It starts a linear motion in one direction as the OP state, continues its translation in the PR state along the inside of an imaginary tube, and then stops abruptly in the OP state as if it hits the bottom of the tube. See Figure~\ref{fig:peeling-group}(b). To specify this abrupt stop, we note the state as OPP instead of OP. 
This {\it AddingKetchup} action group is extracted from the scene of adding ketchup, adding Tabasco, and adding pepper to the food.

In a similar manner, the second variation---a {\it STG3a: CrackingEgg} action—involves cracking the shell of an egg, which abruptly starts as a linear motion and continues the linear motion until it cracks the shell. We can specify this state transition as OPP-PR.

\begin{figure}[h]
    \centering
    \includegraphics[width=\columnwidth]{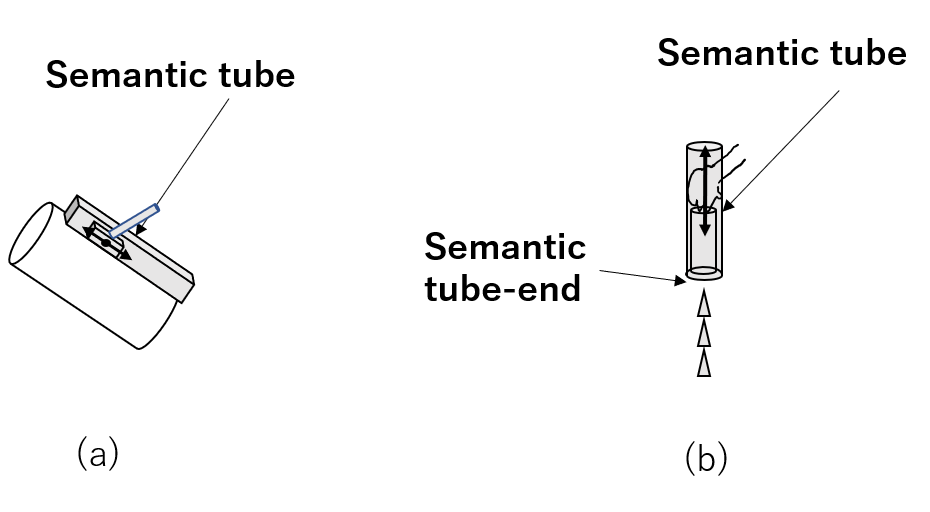}
    \caption{Peeling task group. (a) In a Peeling action, the tool makes a linear translation along a semantic tube parallel to the environment surface.  (b) In an AddingKetchup action, an object makes a sudden stop to sprinkle the contents after a linear motion as if it hits the bottom of the semantic tube.}
    \label{fig:peeling-group}
\end{figure}

\paragraph{STG4: CleaningBowl task group}

A spherical motion occurs when wiping the inner surface of a bowl or pot, i.e, SP rotation state. We refer to this as the {\it Semantic sphere} constraint. Since the center of rotation does not translate during such spherical rotation,
the translation state is maintained as FC state during this action. 

The rotation state transition begins by wiping in one direction, thus is in the OS state.
In the middle of the action, the wiping can rotate
in two axes, thus transitions to the SR state. 
The action finally ends in the OS state. This group is summarized as shown in the semantic sphere row in Table~\ref{tab:semsum}.

\paragraph{STG5: Pouring task group}

This pattern
involves pouring a liquid or semi-liquid from a container to somewhere. Because the container is in the air, the physical condition does not impose any constraint on the container's translation and rotation; it can move or rotate freely under physical constraints. However, for pouring something, it is necessary to fix the flow exit of the container as the FC translation state. The container must rotate about this flow exit. We refer to this as the {\it semantic hinge constraint}. It requires the OR state to start pouring. For avoiding over-pouring, it also requires to OR state at the end. See Figure~\ref{fig:pouring-group}(a).

Another household action that appeared in our data which belonged to this group is a sprinkling action. This action sprinkles something evenly from a hand to a surface, usually a limited place such as the inside of a pan. This action looks slightly different from the pouring action.
However, such an action can be considered as a pure rotation motion with respect to the center of the receiving area, as shown in Fig.~\ref{fig:peeling-group}(b). During this action, the rotation center does not move in the FC translation state. The rotation axis is aligned to the direction perpendicular to the direction of gravity in the RV rotation state. To start the action, the rotation direction is one-way, and to stop, the rotation direction is again one-way. 

The difference between the sprinkling action and pouring action is the location of the rotation center, and the rotation center of a sprinkling action is outside of the convex hull of the hand and the tool, while the pouring action's center is inside of the convex hull. However, in terms of state transition, they belong to the same task group 
{\it STG5}. See Figure~\ref{fig:pouring-group}(b).

\begin{figure}[h]
    \centering
    \includegraphics[width=\columnwidth]{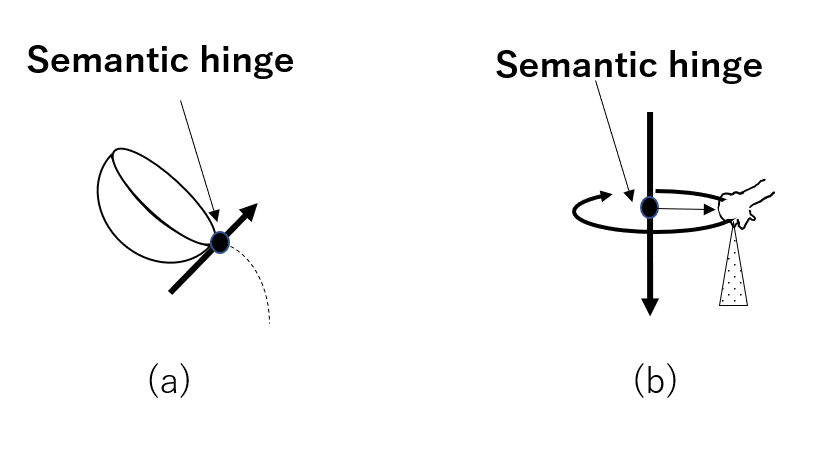}
    \caption{Pouring group. (a) A pouring action rotates an object with respect to an imaginary rotation center owing to the semantic hinge constraint. (b) A sprinkling action rotates an object to sprinkle the contents evenly over the environment based on a rotation motion with respect to the semantic hinge.}
    \label{fig:pouring-group}
\end{figure}

\paragraph{STG6: Holding task group}
When cutting an object, such as a radish, one hand holds a knife and the other hand holds the radish for support. This holding action for supporting can be semantically specified as the FC in translation and FR in rotation, which we refer to as the {\it semantic box constraint}. Thus, 
the object's state stays at
FC in translation and FR in rotation as if 
packed inside of an imaginary box. See \figref{spherebox}(b).

A similar holding action often occurred in the recordings we used, such as 
when collecting running water with a bowl. For this action, the bowl can be rotated around the axis parallel to the direction of gravity while receiving the water. However, 
this was rarely found in our data, thus
we classify these actions semantically to have the FR rotation state. See \figref{spherebox}(a).

\

\begin{figure}
    \centering
    \includegraphics[width=\columnwidth]{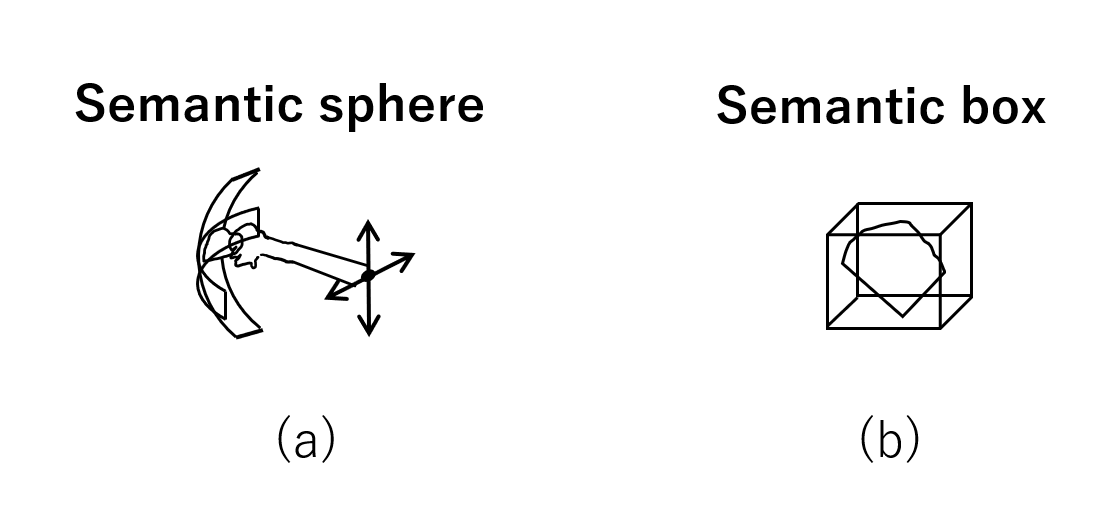}
    \caption{Semantic sphere and semantic box. }
    \label{fig:spherebox}
\end{figure}

\begin{table*}[h]
    \centering
    \caption{STG2 to STG6}
    \label{tab:semsum}
    \begin{tabular}{|l|l||l|l|}
    \hline 
    Semantic constraint & Task Group & Translation & Rotation \\
       \hline \hline
   Semantic wall & STG2: Wiping & OT-TR-OT & \multirow{4}{*}{FR-FR-FR} \\
   \cline{1-3}
   \multirow{3}{*}{Semantic tube} & STG3: Peeling & OP-PR-OP &   \\
    & \ \ STG3a: CrackingEgg & OPP-PR-OP & \\
    & \ \ STG3b: AddingKetchup & OP-PR-OPP & \\
   \hline
   Semantic sphere & STG4: CleaniningBowl & \multirow{3}{*}{FC-(FC)-FC} & OS-SP-OS \\
   \cline{1-2} \cline{4-4}
   Semantic hing & STG5: Pouring &  & OR-RV-OR \\
   \cline{1-2} \cline{4-4}
   Semantic box & STG6: Holding &  & FR-FR \\
       \hline
    \end{tabular}
\end{table*}

\subsection{Multistage class}
Some household actions 
need to be considered as multistage actions.
We analyze such multistage actions; the multistage task groups is abbreviated as MTGs. 

\paragraph{MTG1: Cutting task group}
Actions found in our data such as cutting a radish or a carrot
had the same pattern of being a two-stage action. In the first stage, the knife is translated between two walls in a certain direction, as TR state, until a corner of the knife hits the cutting board. After that, in the second stage, the knife rotates around the contact point until the cutting edge of the knife sits parallel to the cutting board.

The first stage starts from an OT translation state and then
continues its translation motion to transition to
a TR state, 
and finally ends its translation motion 
to an OT state. The rotation is considered to be a 1D rotation as the RV state because the knife 
can rotate with respect to the axis orthogonal to the knife's plane. Thus, the first stage 
has the pattern of
the {\it STG2: Wiping} task group.

In the second stage, the one-way 1D rotation of the knife starts 
in an OR state with respect to the corner of the knife, 
then the knife transitions to a RV state and continues the plane rotation. 
Finally, the knife ends at a OR state 
when the other corner of the knife hits the cutting board. Because the rotation center does not translate during this rotation, the knife has the FC translation state. Thus, the second stage 
has the pattern of
the {\it STG5: Pouring} task group. See Table~\ref{tab:cut}.

\begin{table*}[h]
    \centering
    \caption{Cutting group}
    \label{tab:cut}
    \begin{tabular}{|c|l||c|c|}
    \hline 
     MTG1 & Component task group &  Translation & Rotation \\
       \hline
      \multirow{2}{*}{Cutting} & STG2:Wiping & OT-TR-OT & RV-RV-RV \\
        \cline{2-4}
              & STG5:Pouring & FC-FC-FC & OR-RV-OR \\
        \hline
        
    \end{tabular}

\end{table*}

\paragraph{MTG2: Flipping task group}
A second multistage task group we found was
a flipping action which can be divided into three stages. In the first stage, a 
food, such as a sunny-side egg, is put on a
tool such as a turner, and then both objects depart from the environment, such as a ply pan.
This sequence of transitions is the same pattern as
the \textit{STG11: PickingCarefully} task group. In the second stage, both the food and the tool are rotated in the air,
the same pattern as the \textit{STG4: Pouring task} group. In the final stage,
the two objects land
with the food on the bottom and the tool on the top,
the same pattern as the \textit{STG13: Placingcarefully} task group. See Table~\ref{tab:flip}.

\begin{table*}[h]
    \centering
    \caption{Flipping group}
    \label{tab:flip}
    \begin{tabular}{|c|l||c|c|}
    \hline 
   
     MTG2 & component task group & Translation & Rotation \\
       \hline
       & STG11: PickingCarefully & PC-NC-NC & RV-RV-RV \\
        \cline{2-4}
      Flipping & STG4: Pouring &  FC-FC-FC & OR-RV-OR \\
        \cline{2-4}
       &  STG13: PlacingCarefully & NC-NC-PC & RV-RV-RV \\
        \hline   
    \end{tabular}

\end{table*}

\subsection{Exceptional cases}

We examine to what extent the set of state transitions 
found in the previous section covers the set of all possible transitions that may occur in theory, that is, how much the sufficient set covers the necessary set. 

For the transitions of translation states under translation displacements, 13 possible transitions from the original 10 states given by the Kuhn–-Tucker theory were previously extracted by \cite{ikeuchi1994toward}. This paper groups the original 10 translation states into six states according to their DOFs and provides eight interstate transitions and four intrastate transitions, as shown in Fig.~\ref{fig:state}(a).

For the transitions of rotation states under rotation displacements by proceeding with a similar analysis, the details of which are given in Appendix B, nine interstate transitions and three intrastate transitions are obtained. See Fig. ~\ref{fig:state}(b).

We can label the 
task groups 
composed by the found state transitions,
on the graphs illustrated in Fig. ~\ref{fig:state}. 
Despite our best efforts, some transitions have not been observed in the videos, and these include NC-TR, NC-PR, PC-TR, PC-PR, and TR-PR.

For example, as shown in Figure~\ref{fig:exceptions}(a),  NC-PR corresponds to an action that makes an object, that is originally not constrained at all, to be suddenly constrained from two directions with the environment. For example, a drawer, which is originally fully pulled out, is later pushed in.

For the rotation case, NR-OS, NR-SP, NR-RT, and NR-RV are scenarios wherein the rotation axis, which is originally not constrained at all, is suddenly constrained because of new surface contacts. As shown in the example in Figure~\ref{fig:exceptions}(b), SP-RV, RT-SP,  and RT-RV have similar characteristics.  

These cases are singular cases and were probably not 
found in the YouTube demonstration videos as the videos aim to show standard procedures for the general audience.
In other words, the found task groups are a sufficient set for representing most of the standard household actions involving semantic constraints except for a few actions that may involve singular case state transitions.

\begin{figure*}[h]
    \centering
    \includegraphics[scale=0.8]{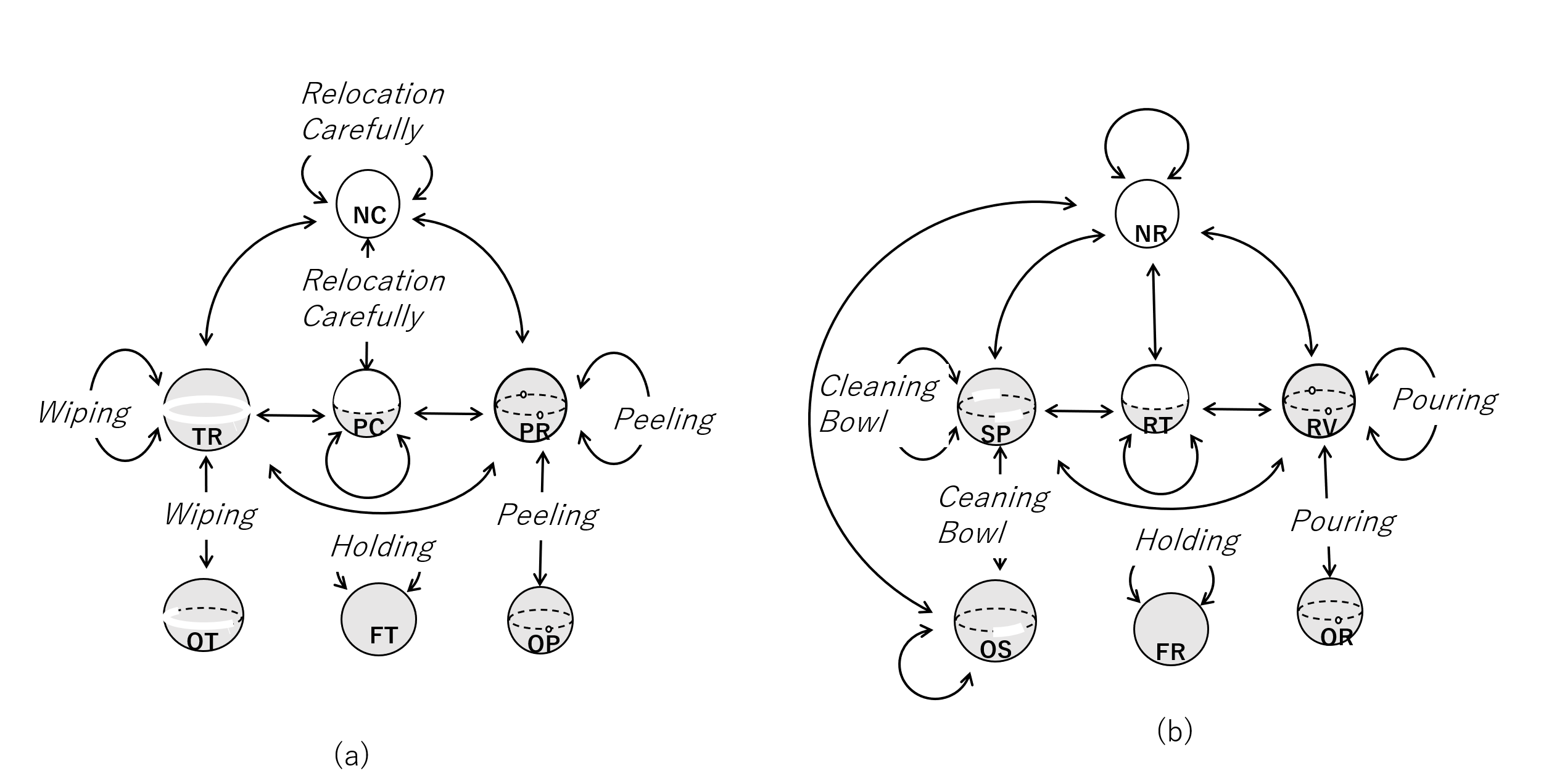}
    \caption{State transition graphs and corresponding household actions (Semantic class only). (a) Translation. (b) Rotation.}
    \label{fig:state}
\end{figure*}

\begin{figure*}
    \centering
    \includegraphics[scale=1.0]{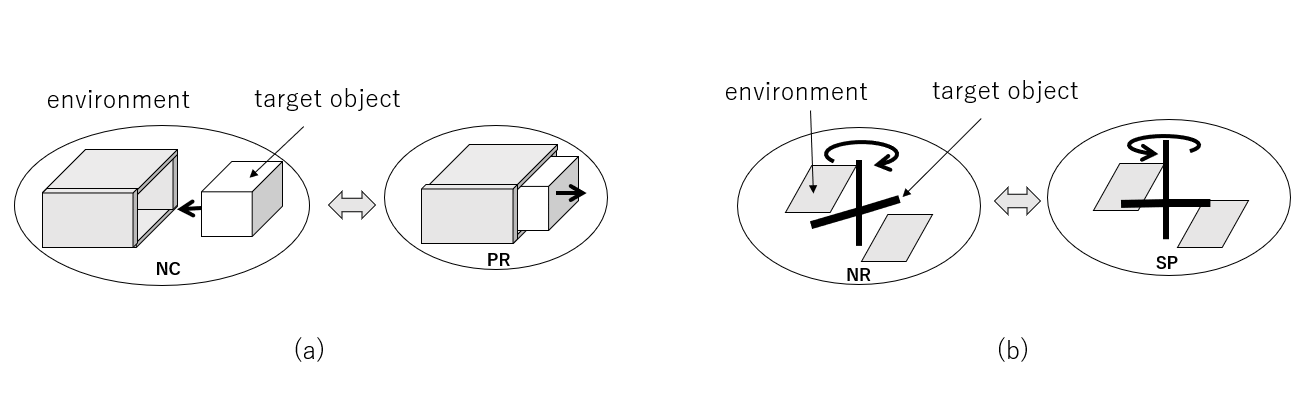}
    \caption{Some cases not extracted from household action videos. (a) {\bf NC-PR} translation transition, (b) {\bf NR-SP} rotation transition.}
    \label{fig:exceptions}
\end{figure*}

\section{Task group recognition from textual instructions}

In the previous sections, we found
a set of task groups that consider physical and semantic constraints and that is capable of representing a sufficient set of household actions. To apply these task groups to robot teaching, a system needs to recognize 
the task groups from human demonstration.

Although the mainstream research on robot teaching is based on visual information such as hand trajectories by \cite{Billard2008,Argall2009a}, visual information is insufficient to decide whether a demonstrated action involves semantic constraints. Some research suggests that verbs and objects are associated with nontrivial information about how objects should be manipulated in terms of motion trajectory or grasp type as in works by \cite{Paulius2019,Wake2020grasp}.
Inspired by these studies, we hypothesize that semantic constraints can be recognized from linguistic information such as verbal or textual instructions and evaluate its performance.

Thus, if the hypothesis is correct, we should also be able to identify the task group (which is a transition of zero or more semantic constraints) by providing verbal or textual instructions during a human demonstration. To test whether this is true, we first collect a dataset of textual instructions, then train a classification model using the dataset, and lastly evaluate the model performance using cross-validation.

\subsection{Textual instructions dataset}

For the dataset,
we focused on basic actions required for room cleaning as according to \cite{Smarr2014}, room cleaning is a household task that the elderlies would like a robot to do for them. To define basic actions, we referred to a list of services offered by a Japanese cleaning agency\footnote{Duskin https://www.duskin.jp/servicemaster/} 
and we selected 20 representative actions.

To collect textual instructions, we first recorded videos wherein one of the authors performed the representative actions. Then, we used a crowd-sourcing service (the Amazon Mechanical Turk) to collect data related to textual instructions by showing the recorded videos.

\subsubsection{Recorded videos}

\figref{sample-video} shows a sample scene in the action video, ``Vacuuming a floor,'' and  Table~\ref{tbl:set-of-actions} contains the 20 titles of our action videos and corresponding task groups. Here, for a bimodal pair of actions, such as pick/place, we only created one video, which corresponds to one action such as only the pick, for simplicity. 

Since differences in video perspectives can affect the content of verbalization, we recorded the first- and third-person perspectives for each action. We ensured that unnecessary information such as unhandled objects are not included in the videos. In total, 40 videos were recorded.

\begin{figure}[h]
    \centering
    \begin{tabular}{c}
    \begin{minipage}{0.5\hsize}
    \centering
    \includegraphics[width=0.9\linewidth]{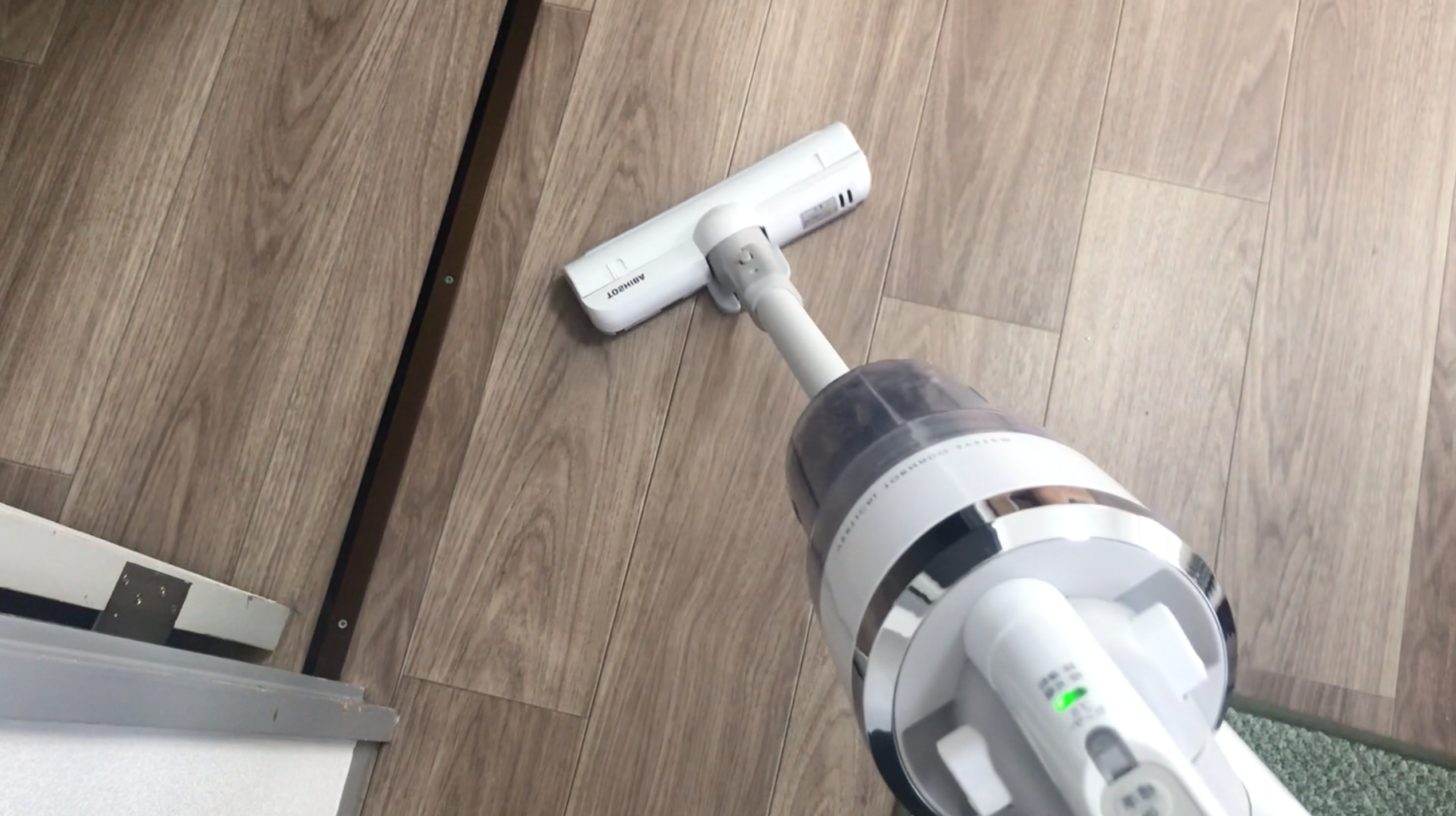}\\
    (a) 
    \end{minipage}
    
    \begin{minipage}{0.5\hsize}
    \centering
    \includegraphics[width=0.9\linewidth]{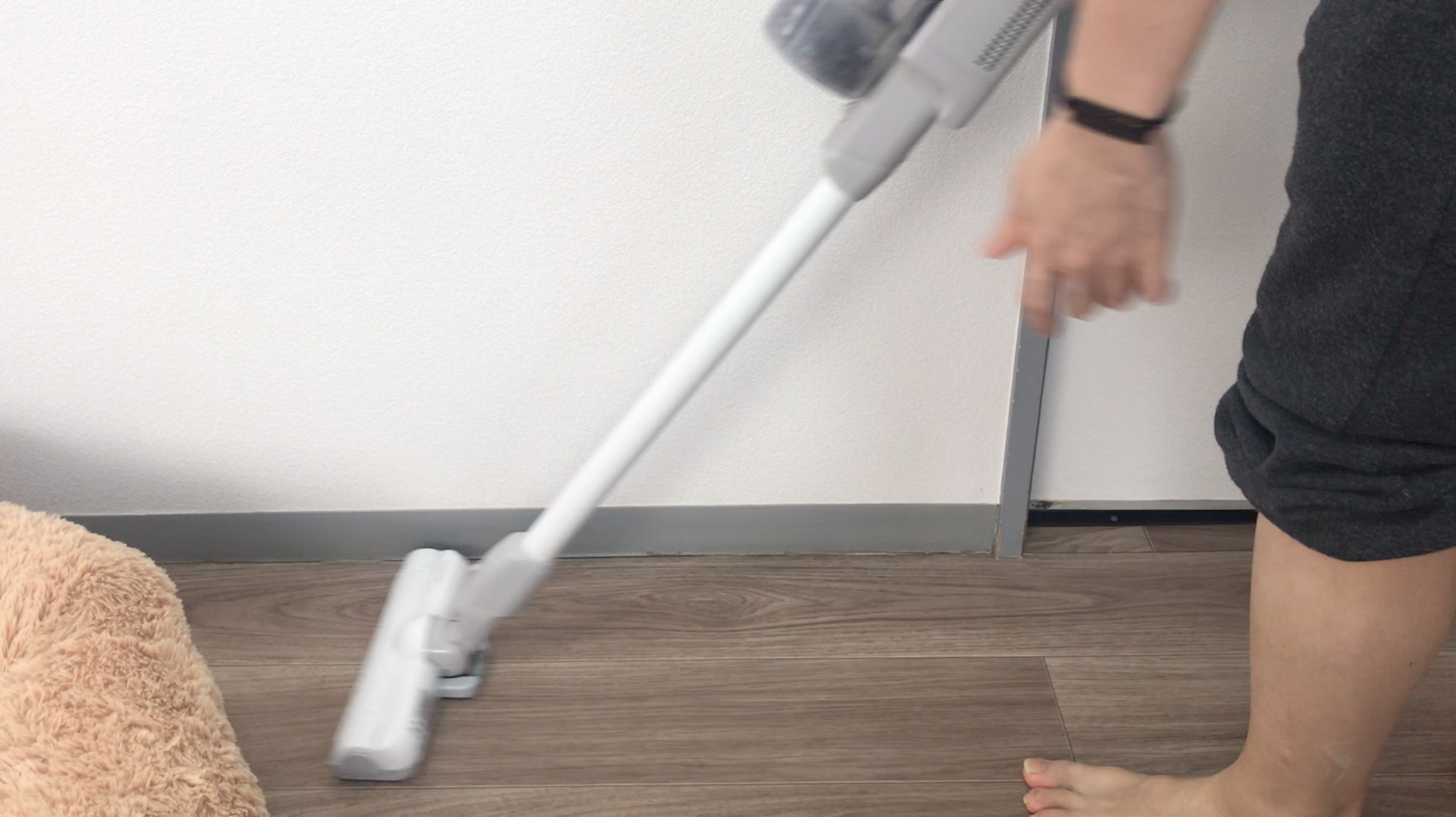} \\
    (b)
    \end{minipage}
    \end{tabular}
    \caption{An action video for \textit{Vacuuming a floor}. (a) First person view. (b) Third person view. }
    \label{fig:sample-video}
\end{figure}

\begin{table}[h]
    \centering
    \caption{Action video titles and corresponding task groups.}
    \begin{tabular}{|l|l|}
        \hline
        Action video & Group \\ \hline
        Picking up a pen from a place. & \multirow{2}{*}{PTG1} \\ \cline{1-1}
        Picking up a small piece of garbage from a floor. & \\ \hline
        Sliding a closed drawer to open. & PTG3 \\ \hline
        Opening a closed fridge door. & \multirow{4}{*}{PTG5}  \\ \cline{1-1}
         Turning a key stuck in a keyhole. & \\ \cline{1-1}
        Rotating a closed door to open. & \\ \cline{1-1}
        Turning a closed faucet. & \\ \hline
        Lifting the cup with its contents. & STG1 \\ \hline
        Clothing the floor with a rag. &  \multirow{6}{*}{STG2} \\ \cline{1-1}
        Moping a floor. & \\ \cline{1-1}
        Spraying on a floor. & \\ \cline{1-1}
        Vacuuming a floor. & \\ \cline{1-1}
        Cleaning a  whiteboard. & \\ \cline{1-1}
        Wiping down a window. & \\ \hline
         Pressing a switch to turn on the lights. & \multirow{4}{*}{STG3} \\ \cline{1-1}
        Opening a curtain. & \\ \cline{1-1}
        Taking clothes from a washing machine. & \\ \cline{1-1}
        Putting a piece of garbage in a trash bin. & \\ \hline
        Hanging a hanger. & \multirow{2}{*}{STG5} \\ \cline{1-1}
       Hanging a hat. & \\ \hline
        \end{tabular}
    \label{tbl:set-of-actions}
\end{table}

\subsubsection{Collecting textual instructions using a crowd-sourcing service}

Prior to data collection, we randomly divided the 20 basic actions into four groups, each of which contained five different actions. Two video sets were then created for each group by randomly selecting one video from the two perspectives for each action. In total, eight video sets were prepared. Ten workers were invited to provide textual instructions for each video set, i.e., by watching five videos. After every video, the workers were prompted to ``describe the task as if you are trying to teach it to somebody else.'' In addition, we asked the workers to provide exactly three different sentences for each video using as diverse vocabularies as possible.
The task duration was set to 20 minutes, and we found all workers completed the task within the time limit. However, a few workers did not provide appropriate answers, for example, filling out all forms with ``no'' or meaningless words, such as ``NA.''
One author manually removed those workers and assigned another worker until ten workers were accepted. 
Several examples of the collected sentences are shown in \tabref{amt-examples}.

\begin{table*}[htb]
    \centering
    \caption{Examples of the collected descriptions for \textit{Vacuuming a floor} using a crowd-sourcing service.}
    \begin{tabular}{|l|l|}
        \hline
         Under first-person view & Under third-person view \\ \hline
            ``Carefully, maneuver the vacuum & ``Someone clean the floor using \\ 
            along the surface of the floor.'' & vacuum cleaner.'' \\ 
            ``This is how you clean the floor &  ``Vacuum the floor frontwards  \\ 
            with a vacuum cleaner.'' & and backwards.'' \\ 
            ``Push the vacuum slowly over & ``Use you  hands and push the \\ 
            all areas of the floor and turn & vacuum in a straight line to \\ 
            it off when finished.'' &  clean the floor.''\\ \hline
        \end{tabular}
    \label{tbl:amt-examples}
\end{table*}

\subsection{Building a classification model of task groups}
We trained a classification model that considers a sentence as an input and outputs the task group of an action corresponding to the sentence.

First, each sentence was vectorized using \cite{mikolov2013distributed}'s word2vec. After lowercasing all the letters, a vector representing a sentence, ${ u }_{ sent.}$, was calculated using the following equation:
\begin{eqnarray}
{ u }_{ sent. } & = & \frac { 1 }{ N } \sum _{ { w }_{ i }\in W }{ { v }_{ i } }, 
\label{eq:sentence-vectorization}
\end{eqnarray}
where ${ N }$, ${ w }_{ i }$, ${ W }$, and ${ v }_{ i }$ represents the number of words in a sentence, ${ i }$-th word from the beginning of the sentence, the set of words that constitute the sentence, and vector representation of ${ w }_{ i }$, respectively. In this study, we used a third-party pretrained word2vec model for vectorization\footnote{https://code.google.com/archive/p/word2vec/}. Words that were not supported by the model were ignored.

The random forests (RF) by \cite{breiman2001random} classifier was used to classify the vectorized sentences . The RF method, which is a type of ensemble learning, can approximate any decision boundary regardless of the linearity of the boundary as explained by \cite{strobl2009introduction}. We used an RF method because we aimed to focus on the classifiability of the sentences into task groups, without assuming the shape of the classification function. The performance of the model was quantified by the performance of cross-validation and the classification accuracy of each action group.

\subsection{Results and discussions}
\figref{confusion} shows the confusion matrix of the prediction where 20\% of the dataset was used for the testing. The result suggests that the model can classify task groups, regardless of PTG and STG, from textual instructions. We conducted five-fold cross-validation 100 times (\figref{validation}) to validate the result. The performance of each cross-validation was calculated as the average accuracy of the five validations. The mean performance across the 100 trials was 77\% (0.92\% standard deviation), which indicated that the model stably classified about 80\% of the sentences in the dataset correctly. These results suggest that 
the task groups
can be recognized from linguistic instructions. 

\begin{figure}
    \centering
    \includegraphics[width=\columnwidth]{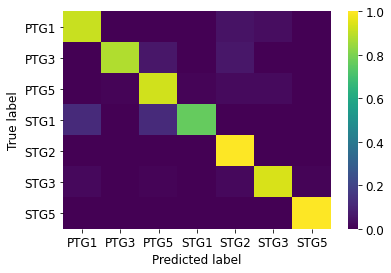}
    \caption{Confusion matrix for recognizing task groups from verbal inputs. Values are normalized by the number of tested samples for each true class.}
    \label{fig:confusion}
\end{figure}

\begin{figure}
    \centering
    \includegraphics[width=\columnwidth]{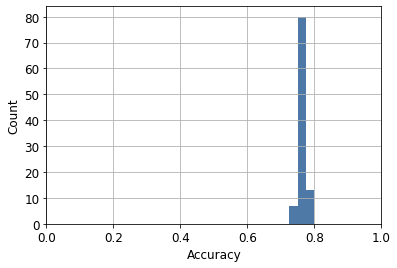}
    \caption{Histogram of the accuracy of five-fold cross validation conducted 100 times.}
    \label{fig:validation}
\end{figure}

This section investigated the possibility of recognizing 
task groups
from human demonstrations. Based on the hypothesis that semantic constraints can be recognized from linguistic information, we created a classification model to estimate task groups from textual instructions. As hypothesized, the model could estimate task groups that include both physical and semantic classes with an accuracy of approximately 80\%. Although the size of the dataset we collected was relatively small, we believe that these results reveal the potential for recognizing
task groups
from human demonstration by analyzing linguistic information with an appropriate dataset and model.

\section{Conclusion}

This paper describes the common sense required to perform household actions by introducing semantic constraints.
Analysis using real home cooking recordings and YouTube videos show that the semantic constraints are capable and appropriate for representing household actions.

An issue with semantic constraints is that they
are difficult, if not impossible, to obtain from observations of human demonstrations.
To demonstrate the possibility of 
recognizing task groups involving semantic constraints from verbal instructions, we conducted a small-scale experiment.
The preliminary experiments allows us to safely conclude that the instructions---even though they vary for each demonstrator---
help recognize the task groups in the house cleaning domain.

In the future, we plan to expand the domain covered, conduct large-scale verbal-task-group matching experiments, and implement these task groups on the learning-from-observation robot.

\section*{Acknowledgement}

The authors thank Prof. Esuko Saito of the home economics department of Ochanomizu University and Prof. Midori Otake of the home economics department of Tokyo Gaukugei University for helpful discussions on elderly supporting techniques. We appreciate their providing miso soap and beef stew making videos.

\bibliographystyle{unsrt}
\bibliography{sample2}

\appendix
\section*{Appendix}

\section{Pure finite rotation and its maintaining DOFs}

The analysis based on infinitesimal rotation has limitations of handling the constraint boundaries that satisfy equality in the screw equation. In particular, $\vec{S}$ satisfies:
\begin{eqnarray}
\vec{M} \cdot \vec{S} & = & 0, 
\label{eq:rotation-zero}
\end{eqnarray}
where $\vec{M}=(\vec{Q} \times \vec{N}) $.

\paragraph{Singular cases}
Before analyzing the general admissible directions of $\vec{S}$, we examine singular cases of the equation.

\begin{itemize}
    
\item $\vec{Q} = 0$. Namely, $\vec{P} - \vec{C}=0$ The contact point $\vec{P}$ is located at the rotation center, $\vec{C}$. Any finite rotations are possible because no finite displacement by the rotation occurs at point, $\vec{P}$.

\item $\vec{M} = 0$. Namely, $\vec{Q}$, the relative vector between the contact point and the rotation center is aligned with the direction as $\vec{N}$. An infinitesimal rotation generates an infinitesimal motion perpendicular to $\vec{N}$. However, a finite rotation generates a circular motion, and further analysis is required. When $\vec{Q} = \vec{N}$, any axis direction generates a displacement against the normal direction except for $\vec{S}=\vec{N}$, and all axes are in the prohibited direction. Further, when $\vec{Q}=-\vec{N}$, all axes directions are admissible.
\end{itemize}

\paragraph{General cases}

Let us consider the general case of one directional contact, that is, Eq.~\ref{eq:rotation-zero}. As shown in Fig.~\ref{fig:gaussian-sphere}(b), the cross-sectional unit circle between the Gaussian sphere and the X-Z plane is the admissible region under the assumption of an infinitesimal rotation. 
 
When the axis direction $\vec{S}$ is on the x–Z plane, the vector is perpendicular to $ \vec{M}$, and the equality holds. In this case, as shown in Fig.~\ref{fig:finite-rotation}, the resulting displacement given by an infinitesimal rotation around the axis is perpendicular to the normal at the point tangential to the local patch. 

However, a finite rotation generates a circular displacement, and the circle comes in contact with the contact point either at the top most point Q2 or the bottom most point Q1, as shown in Fig.~\ref{fig:finite-rotation}(a). If the contact occurs at the bottom most point Q1, a finite rotation generates an admissible detaching displacement from the contact point. If contact occurs at the top-most position, a finite rotation generates a collision displacement to the patch. 

The top and bottom most decisions can be made by examining whether the vector $\vec{S}$ is between the normal direction $\vec{N}$ and the positional vector$\vec{Q}$ using Theorem 1.

\begin{figure}
    \centering
    \includegraphics[width=\columnwidth]{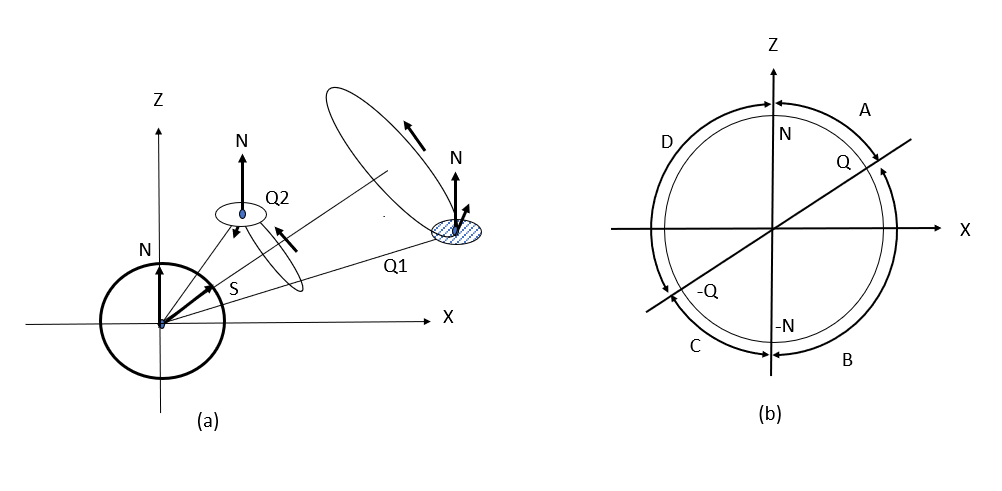}
    \caption{Finite rotation at the tangential points}
    \label{fig:finite-rotation}
\end{figure}

\begin{theorem}

On the cross-sectional circle, when $\vec{S}$ satisfies the condition,  $(\vec{Q} \times \vec{S}) \cdot (\vec{S} \times \vec{N} )  \geq 0 $, the axis direction is the admissible direction.

\end{theorem}

\begin{proof}

We divide the circle into four sections, $A, B, C, and D$ by the points $\vec{N}, -\vec{N}, \vec{Q}, -\vec{Q}$, as shown in Fig.~\ref{fig:finite-rotation}(b). From the relation between the axis and the point, an axis in segments A and C provides the bottom-most contact, while segments B and D provide the top most contact.

In the $A$ segment, a $\vec{S}$ becomes an admissible axis because the rotation circle is tangential at the bottom-most point. In this segment, $(\vec{Q}\times\vec{S}) \cdot (\vec{S} \times \vec{N}) > 0$ holds. In the $B$ segment, a $\vec{S}$ becomes a prohibited axis because the circle is tangential at the top most point. In this segment, 
$(\vec{Q} \times \vec{S}) \cdot (\vec{S} \times \vec{N}) < 0$.
In the $C$ segment, $\vec{S}$ becomes an admissible axis. In this segment,
$(-\vec{Q} \times \vec{S}) \cdot (\vec{S} \times -\vec{N} ) \geq 0$; that is, $(\vec{Q} \times \vec{S}) \cdot (\vec{S} \times \vec{N}) > 0$.
In the $D$ segment, $\vec{S}$ is a prohibited axis. $(-\vec{Q} \times \vec{S}) \cdot (\vec{S} \times \vec{N}) \geq 0$, namely $(\vec{Q} \times \vec{S}) \cdot (\vec{S} \times \vec{N}) < 0$.

The contradiction occurs when a point in segment A or segment C satisfies $ (\vec{Q} \times \vec{S}) \cdot (\vec{S} \times \vec{N}) < 0$. Similarly, when a point in segment C or segment D satisfies $(\vec{Q} \times \vec{S}) \cdot (\vec{S} \times \vec{N}) > 0 $. Therefore, it is necessary for a point in segments A or C to satisfy the condition $ (\vec{Q} \times \vec{S}) \cdot (\vec{S} \times \vec{N}) > 0$ .

When $\vec{S}=\vec{N}$ or $\vec{S}= -\vec{N}$, not only an infinitesimal rotation but also finite rotation provides a displacement along the tangential direction. Thus, this axis direction is included in the admissible direction.

When $\vec{S}=\vec{Q}$ or $\vec{S}=-\vec{Q}$, an infinitesimal rotation and a finite rotation do not generate any displacement at the contact point. Thus, this axis direction is included in the admissible direction.
\end{proof}

\paragraph{One-directional contact}
Let us re-visit one--directional contact on the Gaussian sphere, where the arc between $N$ and $Q$ and the one between $-N$ and $-Q$ is the admissible direction, as shown in Fig.~\ref{fig:admissible-region-one}(a). 
Under the infinitesimal rotation analysis, these arcs represent the maintain directions. In the finite rotation analysis, most arc points become detaching directions. Only four points, $\vec{S} = \vec{N}$, $\vec{S}= -\vec{N}$, $\vec{S} = \vec{Q}$, and $\vec{S} = - \vec{Q}$, provide maintain rotations.
By adding those inner admissible regions, Fig.~\ref{fig:admissible-region-one}(b) depicts admissible directions under one directional contact. 

\begin{figure}
    \centering
    \includegraphics[width=\columnwidth]{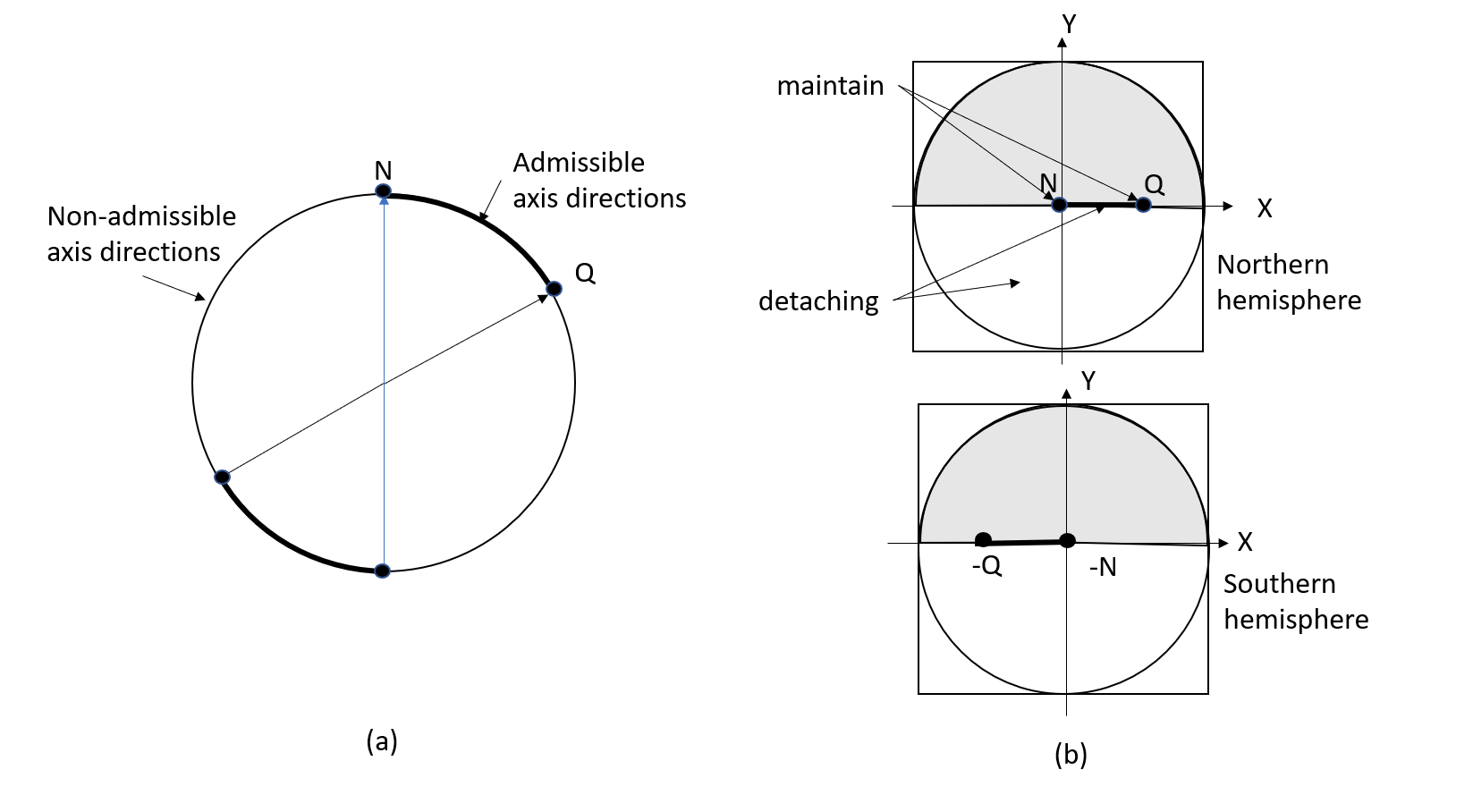}
    \caption{Admissible region of one-directional contact. (a) Admissible directions on the circle. (b) Admissible directions plotted on projected planes. }
    \label{fig:admissible-region-one}
\end{figure}

\paragraph{Two-directional contact (General case)}

Two-directional contact holds two constraint inequality.
\begin{eqnarray}
\vec{M}_1 \cdot \vec{S} & \geq & 0 \nonumber \\
\vec{M}_2 \cdot \vec{S} & \geq & 0
\end{eqnarray}
where $\vec{M}_1  = (\vec{Q}_1 \times \vec{N}_1)$ and $\vec{M}_2  =  (\vec{Q}_2 \times \vec{N}_2)$.
The admissible region on the Gaussian sphere is a crescent region.  

\paragraph{Two-directional contact (Singular case)}

The singular case occurs when $\vec{M}_1 = - \vec{M}_2$. In the infinitesimal rotation, the admissible region is a greater circle perpendicular to $\vec{M}_1$ and $\vec{M}_2$ as shown in Fig.~\ref{fig:two-circle}(a).

The  finite rotation only allows arcs between $\vec{N}_1$ and $\vec{Q}_1$ and between $\vec{N}_2$ and $\vec{Q}_2$ on the greater circle. The admissible arcs, resulting from these two arcs, appear in various patterns depending on the relative relations among those points. For example, in Fig.~\ref{fig:two-circle}(a), admissible regions appear as a pair of arcs, while in Fig.~\ref{fig:two-circle}(c)]. Theorem 2 states that they appear as five patterns: no regions, a pair of points, two pairs of points, a pair of arcs, and two pairs of arcs.

\begin{figure}
    \centering
    \includegraphics[width=\columnwidth]{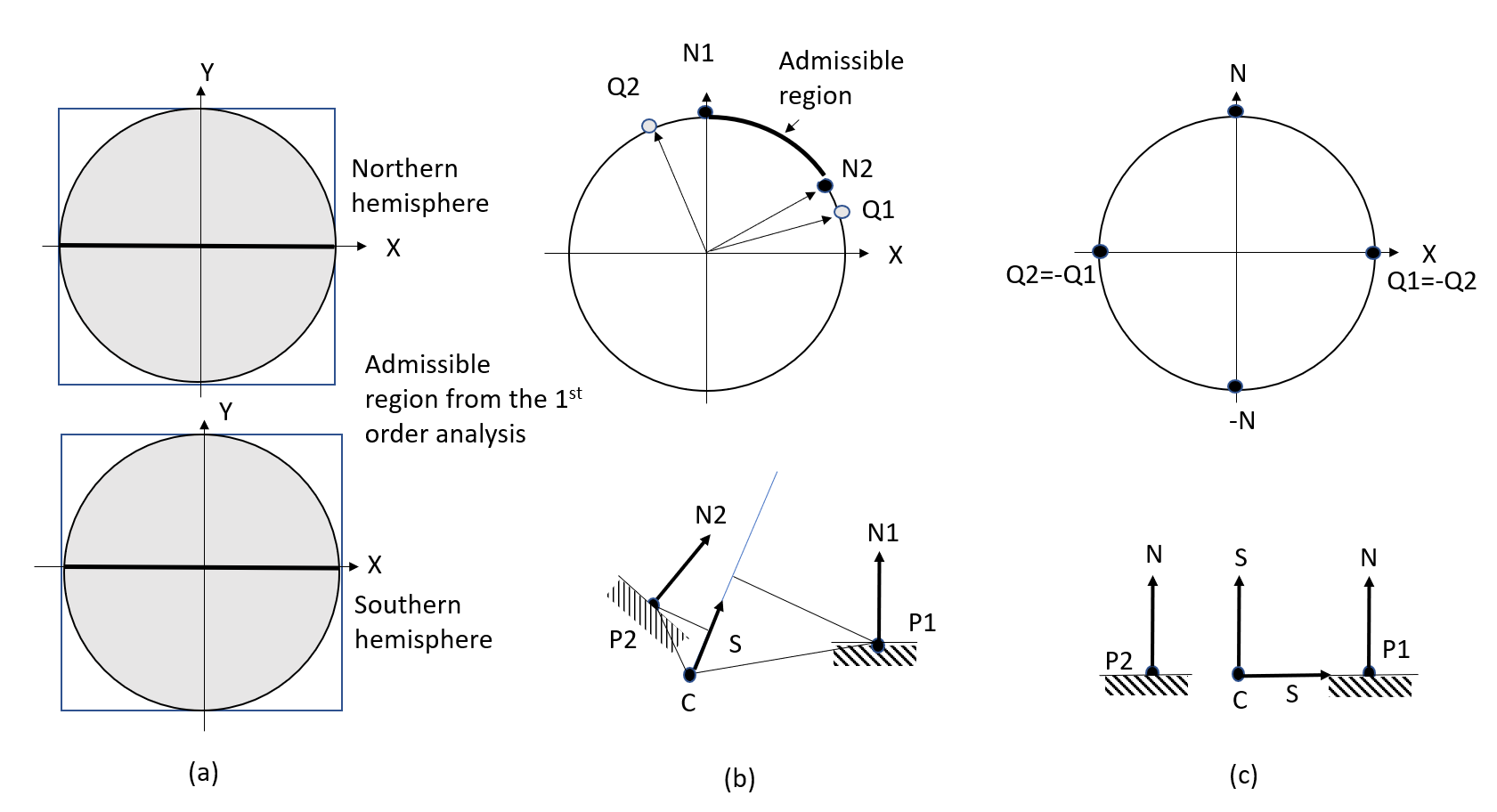}
    \caption{Two-directional contact caused by two opposite constraint vectors, $\vec{M}$ and $-\vec{M}$. (a) First-order analysis. The entire greater circle perpendicular to $\vec{M}$ becomes the admissible region. (b) Second-order analysis of arbitrary positions of contact points. (c) Second-order analysis of the contact points, which have the same surface normal on the same plane, which includes the rotation center. }
    \label{fig:two-circle}
\end{figure}

\begin{theorem}
Two-directional contacts with having $\vec{M}_1=-\vec{M}_2$ appears on the great circle as no region, a pair of points, two pairs of points, a pair of arcs and two pairs of arcs.
\end{theorem}

\begin{proof}
We can set $\vec{N}_1$ as the north pole and $\vec{Q}_1$ on the X-Z plane of the X-positive side, without loss of gene rarity. From this definition, the vectors $\vec{N}_2$ and $\vec{Q}_2$ are on the same plane.

We can use the same labels as segments A, B, C, and D in Fig.~\ref{fig:finite-rotation}(b). We enumerate the following four cases.

Case 1: $\vec{N}_2$ is in segment A. As shown in Fig.~\ref{fig:case1}, $\vec{Q}_2$ can be in segments A, C, or D. This case creates a pair of arcs or two pairs of arcs.

\begin{figure}
    \centering
    \includegraphics[width=\columnwidth]{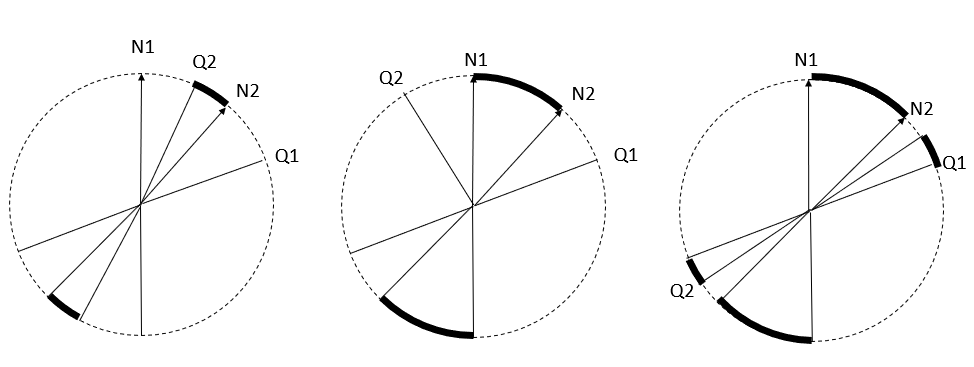}
    \caption{Case1. $\vec{N}$ is in segment A.}
    \label{fig:case1}
\end{figure}

Case 2: $\vec{N}_2$ is in segment B. As shown in Fig.~\ref{fig:case2},this case creates no regions or a pair of arcs.
\begin{figure}
    \centering
    \includegraphics[width=\columnwidth]{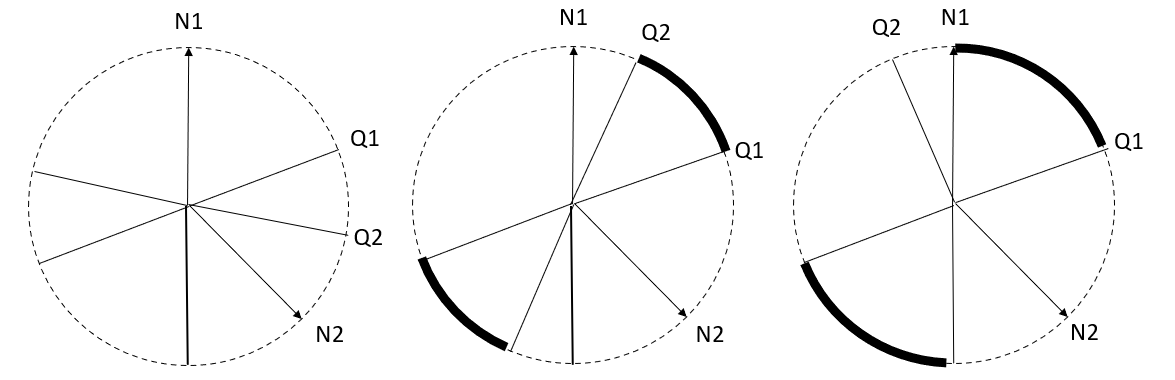}
    \caption{case2.}
    \label{fig:case2}
\end{figure}

Case 3: $\vec{N}_2$ is in segment C. As shown in Fig.~\ref{fig:case3}, this case creates a pair of arcs or two pairs of arcs.

\begin{figure}
    \centering
    \includegraphics[width=\columnwidth]{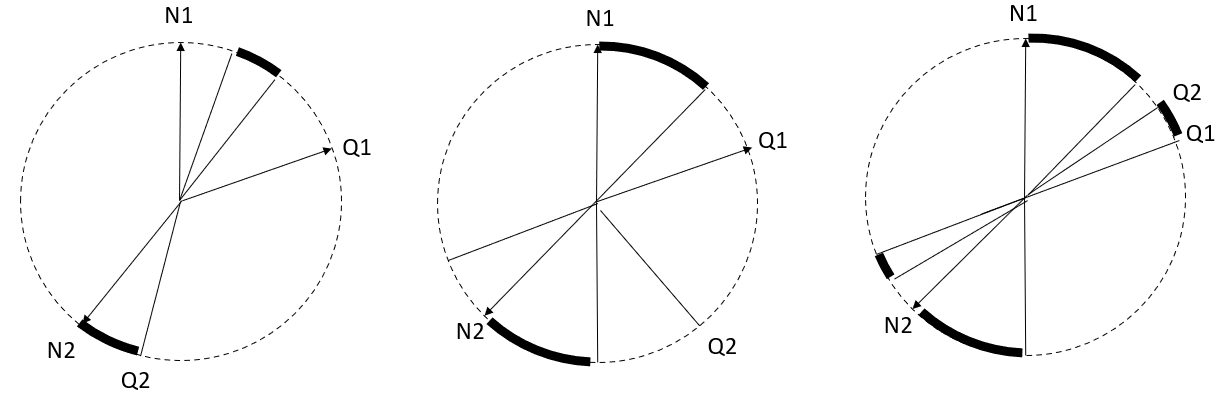}
    \caption{Case3.}
    \label{fig:case3}
\end{figure}

Case 4: $\vec{N}_2$ is in segment D. As shown in Fig.~\ref{fig:case4}, this case creates no regions or a pair of arcs.
\begin{figure}
    \centering
    \includegraphics[width=\columnwidth]{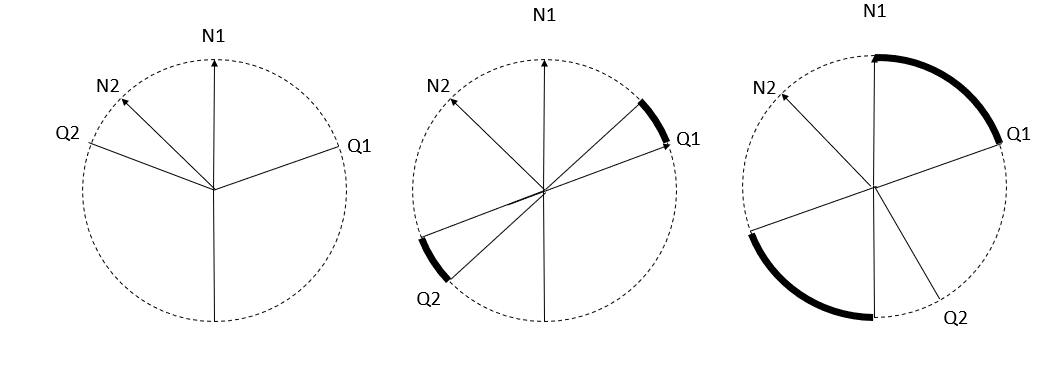}
    \caption{Case4.}
    \label{fig:case4}
\end{figure}

Singular cases occur as shown in Fig.~\ref{fig:case_sing}: two pairs of points or pair of points.

\begin{figure}
    \centering
    \includegraphics[width=\columnwidth]{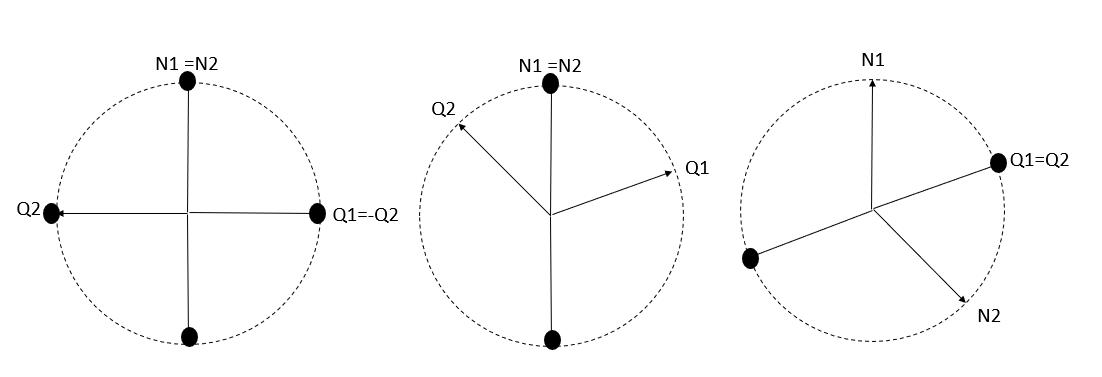}
    \caption{Singular cases}
    \label{fig:case_sing}
\end{figure}

Thus, the patterns are 1) no region, 2) a pair of points, 3) two pairs of points, 4) a pair of arcs and 5) two pairs of arcs.

\end{proof}

\paragraph{Multi-directional contact}

Multi-directional contact adds further contract inequality equations. Each equation provides a hemispherical constraint on a Gaussian sphere. Theorem 3 states that we will have 11 topological patterns in the resulting admissible regions.

\begin{theorem}
Under multi-directional contact, admissible regions form three patterns:
1) polygonal region
2) arc(s), point(s) or combinations of arc(s) and point(s) on a great circle
3) none
\end{theorem}

\begin{proof}

We will prove this by using the induction method.

Under two-directional contact, from Theorem 2, the admissible regions from one of the three patterns:
\begin{itemize}
    \item A polygonal region
    \item Arc(s), point(s) or a combination of these on one great circle
    \item Null region.
\end{itemize}

Under $n-1$ directional contact, let us assume the rank of the coefficient matrix of $N-1$ constraint inequalities as $M-1$. Further, let us assume that the admissible regions form one of the three patterns. Under this assumption, when we add an extra constraint, we will prove that admissible regions form one of the three patterns.

When we add extra inequality, the rank of the coefficient matrix is either $M$ or $M-1$. 

{\it (Case 1: New rank is $M$)}
First, we consider the case when the rank becomes $M$. In this case, the new great circle intersects the previous $M-1$ great circle.

{\it (Case 1-1: Polygonal region}
If the previous admissible region forms a polygonal region, this great circle either intersects this polygonal region or not. When it intersects the region, a new edge appears in this region, which results in a new polygonal region. If not, the original polygonal region remains or disappears; the resulting admissible region is either a polygonal region or none. In any case, the new constraint does not generate any new pattern. The admissible regions form one of the three patterns.

When the previous admissible region forms arc(s), point(s), or a combination on a great circle, the new great circle intersects this great circle. Owing to the inequality of the new great circle, some point(s) and some arc(s) may disappear or some arc(s) may become shorter or one arc becomes a point. In any case, the resulting admissible region forms arc(s), point(s), or a combination of the same great circle.

If the previous admissible region is a null region, a new constraint does not generate any new region, and the resulting admissible region is a null region.

Let us consider the case where the new constraint maintains the rank as $M-1$. 
If the previous admissible region is a polygonal region, the new constraint great circle is either one of the great circles that form the edge of the polygonal region. When the former case occurs, the great circle maintains the polygonal region or converts the polygonal region into arc(s), point(s) or a combination on the new great circle. When the new great circle does not join one of the polygonal edges, the polygonal region disappears or is maintained.

When the previous admissible region is arc(s), point(s), or a combination on a great circle, the new constraint great circle is either the same great circle or not. When it is the same great circle, the constraint either shortens the admissible arc or converts it into a point or reduces the number of arc(s) and point(s) but does not generate a new pattern. When the new great circle is not the same great circle, the new constraint removes some arc(s), point(s), or a combination, but it does not generate a new pattern.

If the previous admissible region is a null region, the new constraint does not generate any new region.

From this discussion, we can conclude that $N$ constraints also form the three patterns.

\end{proof}

In conclusion, we have the following 12 topological patterns of admissible axes on the Gaussian sphere. 
\begin{itemize}
    \item Whole spherical surface -- 3 DOFs
    \item Hemisphere surface -- 2.5 DOFs
    \item Crescent region -- 2.5 DOFs
    \item Polygonal region -- 2.5 DOFs
    \item Combination of arcs and points on a great circle -- 2 DOFs
    \item Pair of arcs on a great circle -- 2 DOFs
    \item Two pairs of points on a great circle -- 2 DOFs
    \item Combination of arcs and points on a great semicircle
    \item Two points on a great semicircle
    \item One arc on a great semicircle -- 1.5 DOFs
    \item Two points on a half great circle --- 1 DOFs
    \item One point -- 0.5 DOFs
    \item Null region -- 0 DOFs
\end{itemize}

\section{State transitions in the rotation states}

In this appendix, state transitions in pure rotation are analyzed in the same way as the transition in pure translation using the disassembly directions~\cite{ikeuchi1994toward}. The following discussion considers the reversibility of the transition and analyzes such transitions in the direction to decrease the contacting environment points. For the rotation of a rigid body, the rotation state differs depending on where the rotation center is located. Therefore, this appendix analyzes such transitions by assuming one particular point as the rotation center with respect to constraining the environment points; the position of the rotation center is preserved before and after the transition. 

\paragraph{RT}

The state RT occurs at one point contact with respect to a rotation center, as shown on the left side of Figure~\ref{fig:RT-NR}, where C denotes the rotation center and A denotes the constraining contact point. By providing a counterclockwise screw at C, as shown in the figure, the contact point A to the environment causes a detaching displacement, it moves to the air in the figure, and it leads to the NR state, as shown on the right side of the figure. Thus, the transition between the RT and NR states exists. By reverse rotation in the NR state, point A hits the environment and the RT state 

\begin{figure}[h]
    \centering
    \includegraphics[width=\columnwidth]{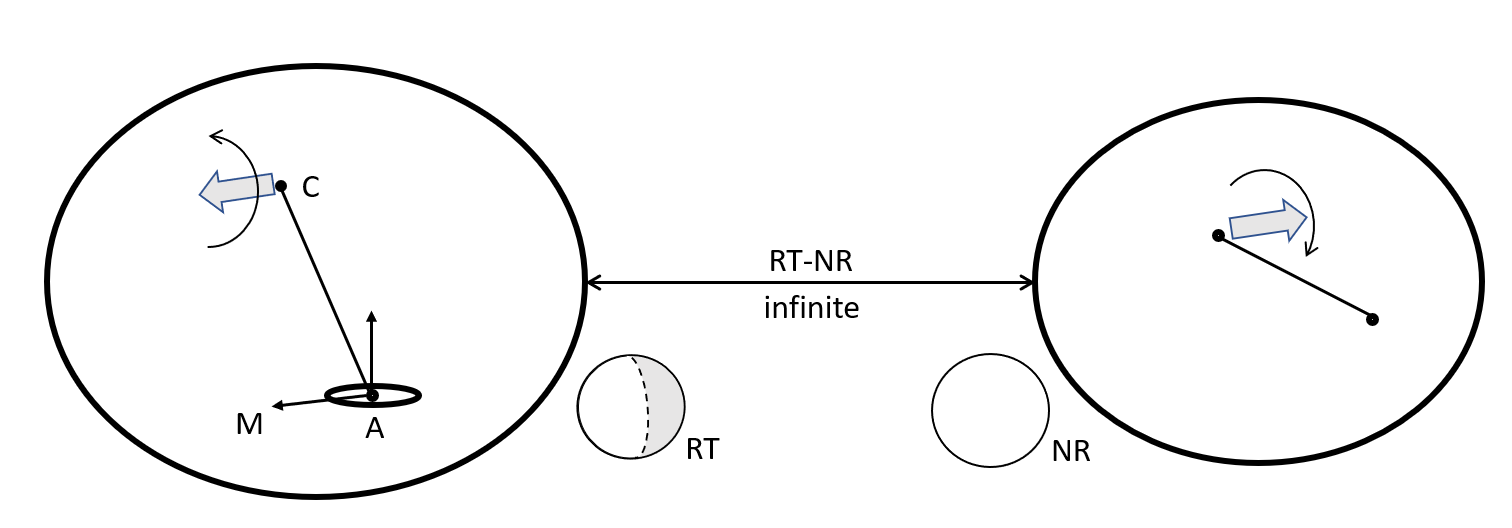}
    \caption{Transition between RT and NR.}
    \label{fig:RT-NR}
\end{figure}

\paragraph{SP}
From the SP state, an infinitesimal rotation does not cause any state transitions. A finite rotation causes two different transitions depending on the shape of the contact surface. The first case occurs, as shown in Figure~\ref{fig:SP-RT}, when the finite rotation makes one contact B to disappear, while it maintains another contact A.

\begin{figure}[h]
    \centering
    \includegraphics[width=\columnwidth]{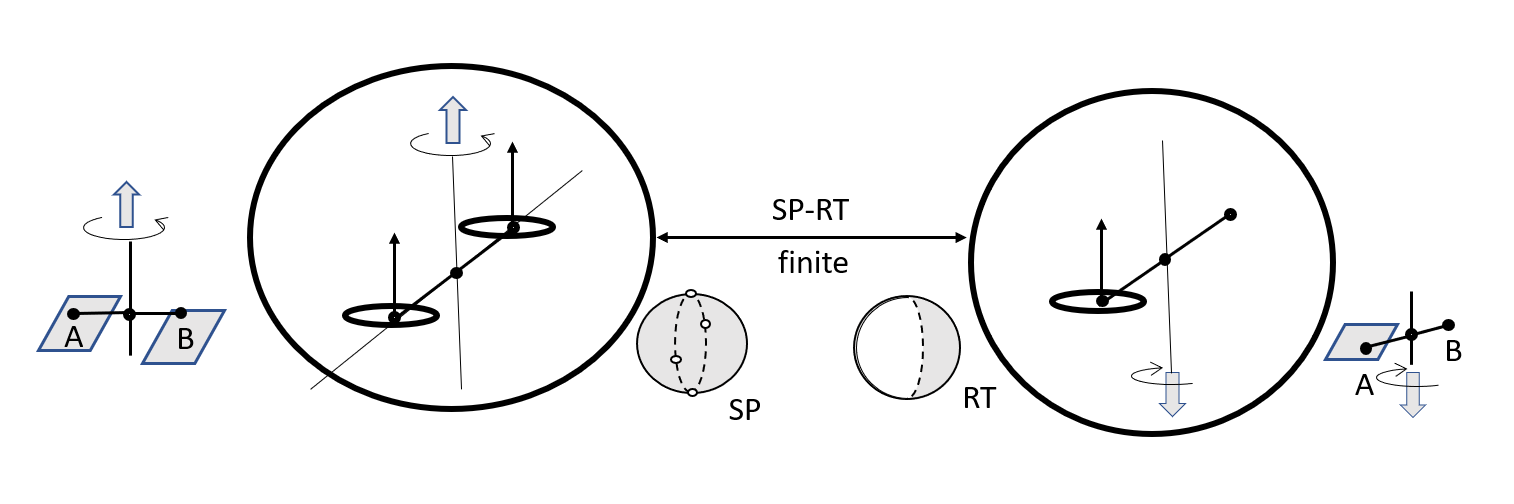}
    \caption{SP-RT transition}
    \label{fig:SP-RT}
\end{figure}

In the same manner, when both contacts A and B disappear simultaneously owing to the shape of the contact environment, a transition from the SP state to the NR state occurs.

\paragraph{OS}
In the OS state, an infinitesimal rotation about the axis through A, B, and C causes a detaching displacement at D; the transition to the PS state occurs. See Figure~\ref{fig:OS-SP}.

\begin{figure}[h]
    \centering
    \includegraphics[width=\columnwidth]{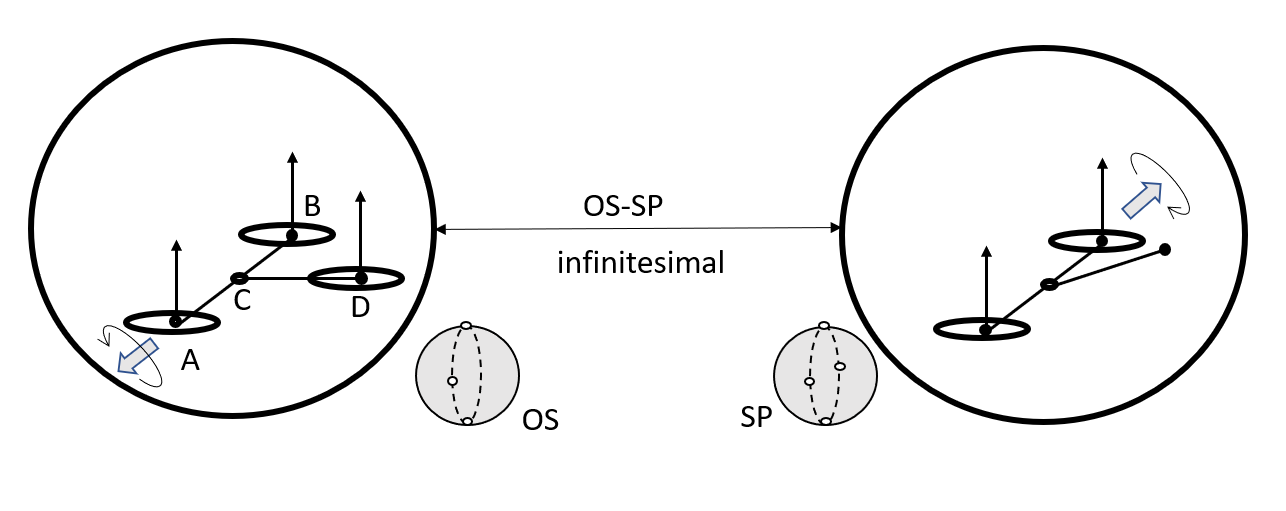}
    \caption{OS-SP transition}
    \label{fig:OS-SP}
\end{figure}

When a finite rotation occurs around the normal direction at the contact points depending on the shape of the environment surface, three different transitions occur. When the contact point C is lost, the transition to the SP state occurs. When the two contact points (for example, B and D) are lost simultaneously, it transitions to RT state. Finally, all three contact points are lost simultaneously, and it transitions to the NR state.

\paragraph{RV}

In the RV state, only finite rotations cause state transitions. Finite rotations of two-point contact, depending on the shape of the environment surface, cause transitions either to the RT or the NR state. Figure ~\ref{fig:RV-RT} shows the transition to the RT state.

\begin{figure}
    \centering
    \includegraphics[width=\columnwidth]{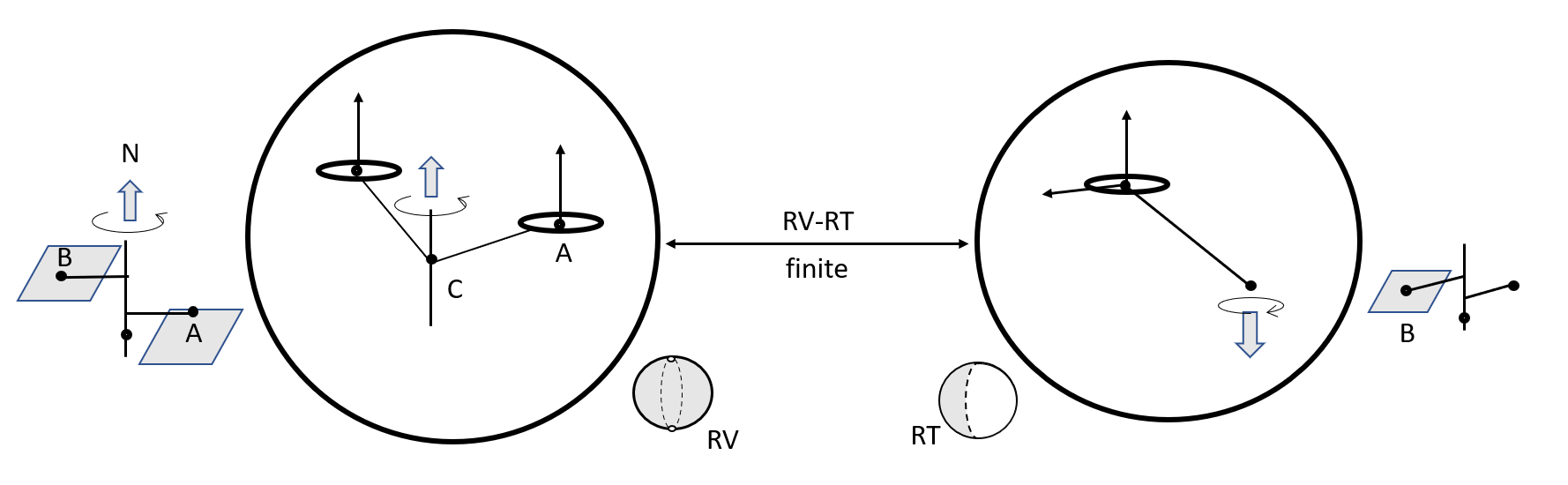}
    \caption{RV-RT transition}
    \label{fig:RV-RT}
\end{figure}

As a singular case, when the RV state is configured with a three-point contact, as shown in Figure~\ref{fig:RV-SP}, the transition to the SP state occurs, and it depends on the shape of the environment.

\begin{figure}[h]
    \centering
    \includegraphics[width=\columnwidth]{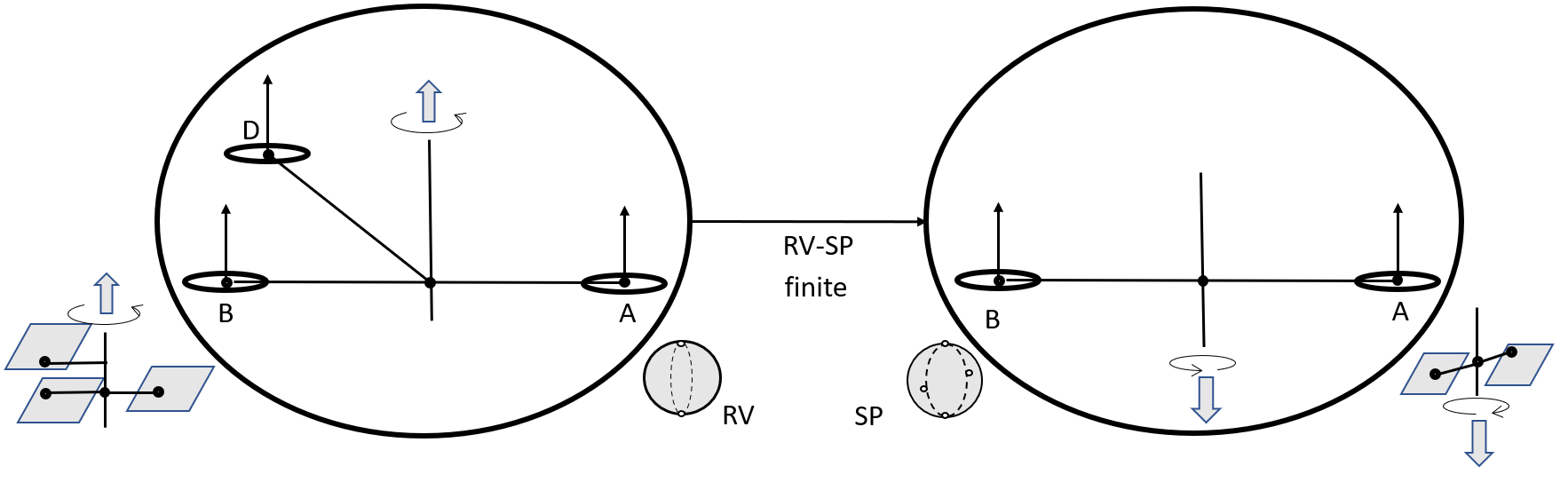}
    \caption{RV-SP transition}
    \label{fig:RV-SP}
    
\end{figure}

\paragraph{OR}
From the OR state, an infinitesimal rotation, which causes a detaching displacement at A, is the only possible rotation. Once this rotation occurs, the transition to the RV state occurs, as shown in Figure\_ref{fig:OR-RV}.

\begin{figure}
    \centering
    \includegraphics[width=\columnwidth]{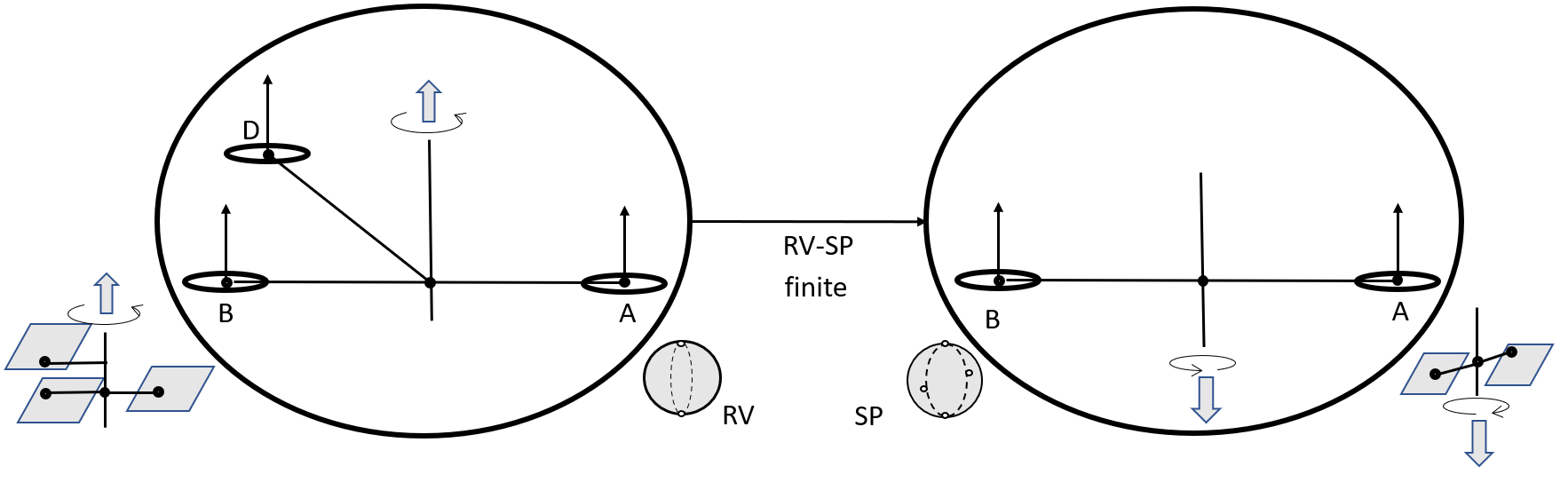}
    \caption{OR-RV transition}
    \label{fig:OR-RV}
    \end{figure}
    
\paragraph{State transitions in rotation}
From the above discussion, a transition graph for rotation can be obtained, as shown in Figure~\ref{fig:state-transition-graph}. As shown in the figure, when maintain rotations exist, such self-state transitions are also included. 

\begin{figure}
    \centering
    \includegraphics[width=\columnwidth]{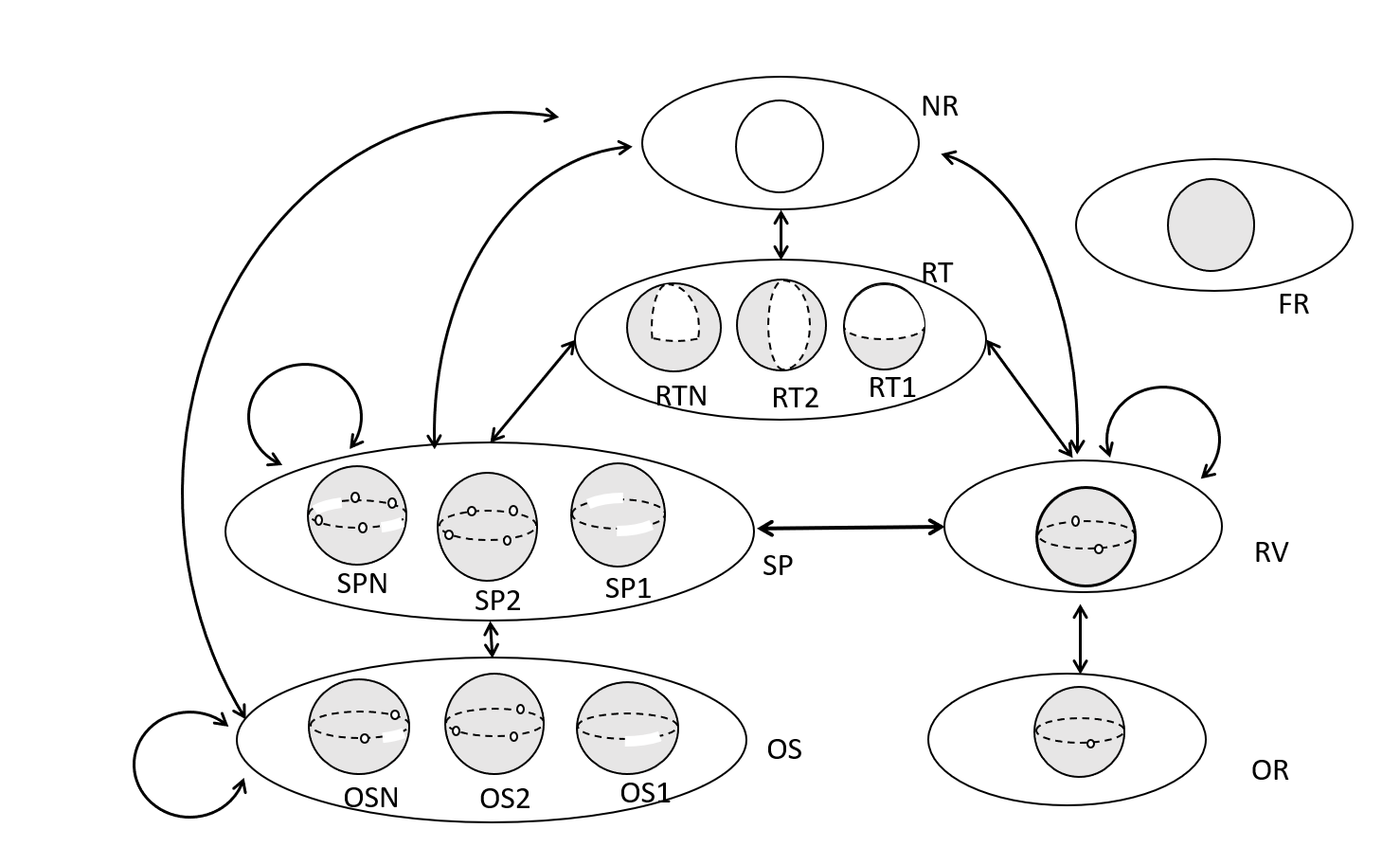}
    \caption{State transitions in rotation}
    \label{fig:state-transition-graph}
\end{figure}

\end{document}